\pdfoutput=1

\pdfoutput=1

\documentclass[11pt]{scrartcl}
\usepackage[a4paper, total={16cm, 24cm}]{geometry}
\usepackage{microtype} 
\usepackage{authblk}
\title{Picking a Representative Set of Solutions in Multiobjective
Optimization: \\ Axioms, Algorithms, and Experiments}
\date{\vspace{-1.5cm}}

\author[ ]{Niclas Boehmer}
\author[ ]{Maximilian T. Wittmann}
\affil[ ]{Hasso Plattner Institute, University of Potsdam, Germany} 
\affil[ ]{\texttt{niclas.boehmer@hpi.de, maximilian.wittmann@hpi.de}} 

\usepackage{xcolor}
\usepackage{graphicx}
\usepackage{pgfplots}
\pgfplotsset{compat=1.17}
\usepackage{csquotes}
\usepackage{mathtools}
\usepackage{amssymb}
\usepackage{amsthm}
\usepackage{thmtools}
\usepackage{nicefrac}
\usepackage{pifont}
\usepackage{dsfont}
\usepackage{microtype}

\newcommand{\cmark}{\ding{51}}%
\newcommand{\xmark}{\ding{55}}%
\DeclareMathOperator*{\argmin}{arg\,min}
\newcommand{\R}{\mathbb{R}}
\newcommand{\N}{\mathbb{N}}
\newcommand{\Z}{\mathbb{Z}}
\newcommand{\veps}{\varepsilon}

\newcommand{\UPP}{\textsc{Uniformity Pareto Pruning}}
\newcommand{\DCPP}{\textsc{Directed Coverage Pareto Pruning}}

[sibling=theorem]
\declaretheorem{lemma, proposition, corollary}[style=plain, sibling=theorem]
[style=plain, sibling=theorem]
[style=plain, sibling=theorem]
[sibling=theorem]

[style=plain, sibling=theorem]

\usepackage{multirow}
\usepackage{booktabs}  
\usepackage{standalone}

\usepackage[pagebackref]{hyperref}
\hypersetup{
	pdfencoding=auto,
	psdextra,
	colorlinks=true,
	citecolor=green!50!black,
	linkcolor=red!60!black,
}

\usepackage[nameinlink]{cleveref}

\usepackage{natbib}
\usepackage{algorithm}
\usepackage[noend]{algorithmic}

\usepackage[inline]{enumitem}
\usepackage{booktabs}
\usepackage{subcaption}

\usepackage[textsize=tiny]{todonotes}

\newcommand{\restatehere}[1]{%
	\marginline{\vspace{0.6cm}\footnotesize \hyperlink{original#1}{\hypertarget{restated#1}{[Main]}}}%
	\csname #1\endcsname*%
}

\begin{document}

\maketitle

\bigskip
{\footnotesize\tableofcontents}

\newpage 

\begin{abstract}
\begin{center}
		\textbf{\textsf{Abstract}} \smallskip
	\end{center}
  Many real-world decision-making problems involve optimizing
  multiple objectives simultaneously, rendering the selection of the
  most preferred solution a non-trivial problem: All Pareto optimal
  solutions are viable candidates, and it is typically up to a
  decision maker to select one for implementation based on their
  subjective preferences.
  To reduce the cognitive load on the decision maker, previous work
  has introduced the Pareto pruning problem, where the goal is to
  compute a fixed-size subset of Pareto optimal solutions that best
  represent the full set, as evaluated by a given quality measure.
  Reframing Pareto pruning as a multiwinner voting problem, we
  conduct an axiomatic analysis of existing quality measures,
  uncovering several unintuitive behaviors. Motivated by these
  findings, we introduce a new measure, \emph{directed coverage}.
  We also analyze the computational complexity of optimizing various
  quality measures, identifying previously unknown boundaries between
  tractable and intractable cases depending on the number and
  structure of the objectives. Finally, we present an experimental
  evaluation, demonstrating that the choice of quality measure has a
  decisive impact on the characteristics of the selected set of
  solutions and that our proposed measure performs competitively or
  even favorably across a range of settings.
\end{abstract}

\section{Introduction}

Many real-world decision-making problems in domains such as systems
design, engineering, operations research, and healthcare are
inherently multiobjective
\citep{DBLP:conf/dagstuhl/StewartBBCEGJNPS08,marler2004survey,DBLP:journals/anor/EriskinKZ24}.
As a result, multiobjective optimization has become a central
research area
\citep{DBLP:conf/dagstuhl/2008moo,ehrgott2005multicriteria}, and
multiobjective variants of many classical algorithmic techniques,
including reinforcement learning
\citep{DBLP:journals/aamas/HayesRBKMRVZDHH22}, integer programming
\citep{DBLP:journals/eor/SylvaC07}, scheduling
\citep{DBLP:journals/candie/GuoWLR13}, and flows
\citep{DBLP:journals/cor/EusebioFE14}, have been intensively studied.

A key challenge in multiobjective optimization is the absence of a
single, objectively best solution. Instead, all solutions on the
Pareto front, i.e., solutions that are not (weakly) outperformed by
another solution in every objective, are viable options for
implementation, with each of them reflecting a different tradeoff
among the objectives. To resolve this, the multiobjective literature
assumes the presence of a decision maker (DM) who selects a final
solution to be implemented based on their subjective preferences.
Herein, a canonical approach is to first compute the Pareto front and
then present it to the DM for selection (see \Cref{sec:RW} for a
  discussion of alternative approaches and additional background on
multiobjective optimization). However, in practice, the Pareto front
is often very large, making it cognitively infeasible for the DM to
process all solutions and compare them effectively.
This motivates the study of the \emph{Pareto pruning} problem (also
known as the representation problem):
Compute a fixed-size subset of Pareto optimal solutions that
represents the overall structure and available tradeoffs of the full
Pareto front well
\citep{DBLP:journals/cor/VazPFKS15,DBLP:journals/candie/PetchrompoCBWP22,taboada2007data,DBLP:journals/eor/PetchrompoWP22,DBLP:journals/mp/Sayin00,DBLP:journals/eor/ZioB11,DBLP:journals/ress/WangLFLC20,DBLP:journals/ress/TaboadaBCW07,DBLP:journals/cor/EusebioFE14}.
A wide range of quality measures have been proposed to evaluate the
selected subset
\citep{DBLP:journals/csur/LiY19,faulkenberg2010quality}, each inducing
a different solution method by selecting the subset optimizing the measure.
Two widely used measures are \emph{uniformity}, which aims to
maximize the minimum distance between any two selected solutions, and
\emph{coverage}, which minimizes the maximum distance from any
non-selected solution to its nearest selected neighbor
\citep{DBLP:journals/mp/Sayin00}.

Despite the wide range of studied measures, systematic comparisons of
their formal properties and a comprehensive analysis of their
computational complexity remain largely absent from the literature.
Existing comparative work (e.g., \citet{DBLP:journals/csur/LiY19} and
\citet{faulkenberg2010quality}) typically groups quality measures
into different categories based on soft criteria or compares them
experimentally. Prior algorithmic work has mostly considered the case
of two objectives \citep{DBLP:journals/cor/VazPFKS15} or has focused
on heuristic and evolutionary methods
\citep{taboada2008multi,DBLP:journals/candie/PetchrompoCBWP22}.

For our analysis, we approach the Pareto pruning problem through the
lens of social choice theory, also known as the theory of collective
decision making.
One of the most actively studied problems in social choice is
multiwinner voting, where the goal is to select a subset of $k$
candidates based on the preferences of voters over a given set of
candidates. We observe a natural analogy: solutions to the
multiobjective optimization problem correspond to candidates, and
each objective constitutes a voter, who evaluates the solutions
according to their performance in the respective objective.
This connection lends itself to an axiomatic analysis of quality
measures, a core method in social choice theory, to enable a
structured comparison between them. Second, the common distinction
between ordinal, cardinal, and approval preferences in social choice
motivates an analysis of the Pareto pruning problem under analogous
assumptions about the structure of the objectives, as ordinal and
approval objectives are generally easier to elicit, especially if
objectives correspond to different human evaluators.

In addition to contributing to the multiobjective optimization
literature, this work also offers a new perspective on multiwinner
voting that complements the three classical paradigms of
\emph{individual excellence}, \emph{proportional representation}, and
\emph{diversity} \citep{faliszewski2017multiwinner}.
The goal of the Pareto pruning problem is distinct from these three
in that the focus lies on ``satisfying'' the candidates, i.e.,
solutions, and not the voters, i.e., objectives. In Pareto pruning,
voters are merely used to assess the similarity between two
candidates, i.e., two candidates are considered close if they are
evaluated similarly by all voters.
Uniformity then seeks a set of mutually dissimilar candidates, while
coverage aims to ensure that every non-selected candidate is close to
at least one selected candidate.
To the best of our knowledge, the only prior work in the voting space
adopting a somewhat similar perspective is that of
\citet{DBLP:conf/ijcai/DelemazureJ0S24}, who, in the spirit of
uniformity, consider the problem of selecting two distant candidates.

\subsection{Our Contributions}
We present a systematic study of quality measures for Pareto pruning
in multiobjective problems, taking a holistic perspective by
integrating axiomatic, algorithmic, and experimental analyses.
Our goal is to contribute formal structure and arguments to the
discussion around quality measures that have traditionally remained
quite fragmented and disconnected across different approaches.
Our analysis connects three traditionally distinct areas:
multiobjective optimization, social choice, and computational geometry.
Besides considering $\ell_1$-variants of the widely studied
uniformity and coverage measures, we also introduce a novel measure,
directed coverage.

In \Cref{sec:ax}, we discuss different desiderata for Pareto pruning
and initiate the axiomatic analysis of quality measures. We propose
five axioms capturing whether solutions reflecting distinct tradeoffs
are guaranteed to be selected  (standout consistency and outlier
consistency), and how the selected subset responds to changes in
solutions' performance or the addition of new solutions  (extremism
monotonicity, monotonicity, and $\veps$-split-proofness). An overview
of which measures satisfy which axioms is provided in \Cref{tab:axioms}.
Motivated by the shortcomings of uniformity and coverage revealed in
our axiomatic analysis, we introduce \emph{directed coverage}, a
measure that, like coverage, aims to cover all solutions well, but,
unlike coverage, evaluates how well a solution $a$ covers another
solution $b$ not by their $\ell_1$-distance but by the summed extent
to which $b$ outperforms $a$.
Unlike uniformity and coverage, directed coverage guarantees, for
instance, that a selected solution will continue to get selected in
case it improves its performance.

In \Cref{sec:alg}, we conduct a thorough algorithmic analysis of
computing the optimal pruning under the three considered quality
measures (see \Cref{tab:alg}).  \citet{DBLP:journals/cor/VazPFKS15}
established that Pareto pruning for uniformity and coverage is
solvable in polynomial time for two objectives, but left the
complexity for more than two and even an arbitrary number of
objectives open.\footnote{This focus on few objectives reflects that
  many classical multiobjective problems involve only two to four
  objectives
\citep{marler2004survey,DBLP:conf/dagstuhl/2008moo,ehrgott2005multicriteria}.}
Building on results from computational geometry, we prove NP-hardness
for uniformity and coverage for three objectives, thereby identifying
the precise boundary of tractability. Along the way, we present a
proof for the NP-hardness of the classic \textsc{Discrete $k$-Center}
problem for the $\ell_1$-distance in two dimensions, which has
surprisingly been missing from the computational geometry literature.
We further extend the algorithmic analysis to our new  \emph{directed
coverage} measure and explore variants of the problem for different
practical restrictions on the type of information provided by each
objective, i.e., besides cardinal (score-based) objectives, we also
explore ordinal (ranking-based) and approval (binary) ones.
While we are unable to observe a difference in computational
complexity when moving from cardinal to ordinal objectives, we find
that approval objectives render all three pruning problems solvable
in polynomial time for any constant number of objectives.

In \Cref{sec:exp}, we conduct an experimental analysis of our three
considered measures, observing that each yields distinctly different
results, and that optimizing for directed coverage introduces a new
perspective,  resulting in the selection of slightly more efficient solutions.

Our code and additional experimental results are available at \url{https://github.com/maxitw/picking_representative_moo}.

\section{Preliminaries}

For some $n\in \N$, let $[n]:=\{ 1, \dots, n\}$.
In a $d$-dimensional multiobjective optimization problem, a finite
set $X$ of alternatives is evaluated by $d \in \N$ \emph{objective
functions} $f_i:X \to \R$ for $i \in [d]$, where  $f_i(x) < f_i(y)$
for two alternatives $x$ and $y$ means that $y$ \emph{outperforms}
$x$ under the $i$-th objective.
The overarching goal of a multiobjective optimization problem is to
maximize all objectives simultaneously, that is, to analyze
$\max_{x \in X} (f_1(x), \dots, f_d(x))$.
We write $f:X \to \R^d, f(x) = (f_1(x), \dots, f_d(x))$ for the
function $f$ aggregating all objectives into the \emph{objective space} $\R^d$.
For two alternatives $x,y \in X$, we say that $x$ is \emph{dominated}
by $y$ if $f_i(x) \leq f_i(y)$ for all $i \in [d]$ and there exists
$j \in [d]$ with $f_j(x) < f_j(y)$. In addition, we say $x$ is
\emph{Pareto dominated} if there exists some $y\in X$ such that $x$
is dominated by $y$. Otherwise,  we call $x$ \emph{Pareto optimal}.

For $i\in [d]$,  an objective $f_i$ is called an \emph{approval
objective} if $f_i(x)\in \{0,1\}$ for all $x\in X$, i.e., each
alternative is either approved or disapproved by the objective, and
an \emph{ordinal objective} if $f_i$ is a bijection from $X$ to
$[|X|]$, i.e., $f_i$ arranges all alternatives from $X$ in a strict
ranking. We refer to the general, unrestricted case as a
\emph{cardinal objective}.

Pareto dominated alternatives are of little importance to a DM, since
there is a strictly better option available.
Accordingly, we will only implicitly assume the existence of $X$ and
instead operate directly on the set of Pareto optimal alternatives,
i.e., we ``preprocess'' our instances to only include Pareto optimal
alternatives.
Similarly, we will only implicitly assume the existence of $f_i$ and
instead treat each alternative as a point in $\R^d$ with its $i$-th
component denoting its value according to $f_i$.
Formally, as input to our problem, we receive the \emph{set of Pareto
optimal alternatives} $A = \{f(x) \mid x\in X \wedge x \text{ is
Pareto optimal}\} \subseteq \R^d$, to which we will refer as
\emph{alternatives} for short.
Our goal is to ``inform'' the DM about $A$ by selecting $k$
alternatives from $A$ for some given $k\in \mathbb{N}$.
We call a subset $S \subseteq A$ with $|S| = k$ a \emph{slate}.

To measure the similarity between two alternatives $x,y \in A$, we
use the Manhattan norm, also known as $\ell_1$-norm as $ || x-y || =
\sum_{i = 1}^d |x_i-y_i|$. Intuitively, two alternatives that are
close to each other present similar tradeoff decisions to the DM.  We
further introduce a ``directed'' variant of the Manhattan norm $||x -
y||_+=\sum_{i = 1}^d \max(x_i - y_i, 0)$. Note that $||\cdot||_+$ is not a metric, as it is not symmetric.
 I.e., we generally have $||x - y||_+\neq ||y - x||_+$.

We will use different \emph{measures} to evaluate the quality of a
slate $S\subseteq A$.
We refer to a generic measure as $I$, where it will always be clear
from context whether lower or higher values of $I$ are preferable.
For some set of alternatives $A$ and integer $k$, we let
$\mathcal{S}(I,A,k)$ be the set of slates which are optimal according
to measure $I$, i.e., subsets $S\subseteq A$ with $|S| = k$ that
maximize (resp.\ minimize) the value of $I$.

\section{Pareto Pruning: Problem, Quality Measures, and Axioms}\label{sec:ax}
We present a general formulation of the Pareto pruning problem in
\Cref{sec:PS}, the three quality measures we examine in
\Cref{sub:QM}, and our axiomatic analysis in \Cref{sub:axAn}.
\subsection{Problem Setting and Desiderata}\label{sec:PS}
We study the Pareto pruning problem, where given a set of
alternatives $A$ and an integer $k$,  we want to select a size-$k$
slate $S\subseteq A$ (to be presented to a DM). Three natural
desiderata for the selected slate $S$, regularly discussed in the
literature under potentially different names
\citep{DBLP:journals/candie/PetchrompoCBWP22,DBLP:conf/dagstuhl/2008moo,DBLP:journals/csur/LiY19},~are:
\begin{description}
  \item[Diversity] $S$ should be ``redundancy-free'', i.e., no two
    selected alternatives should be similar to each
    other.\footnote{Note that the term ``diversity'' is quite
      overused in the multiobjective literature and sometimes also
      refers to what we call representativity. Our notion of diversity
      is also distinct from the notion of diversity of
      \citet{faliszewski2017multiwinner} from the multiwinner voting
      literature, as their notion captures the idea of selecting
      alternatives so that for each objective, there is at least one
      alternative that is evaluated highly by this objective. In
      contrast, our notion of diversity is in line with the diversity
      notion used in a recent line of works in artificial intelligence,
      where the goal is to compute a set of sufficiently distinct
      solutions to a problem
    \citep{DBLP:conf/aaai/ArrighiFO023,DBLP:journals/ai/BasteFJMOPR22,DBLP:conf/aaai/HebrardHOW05,DBLP:conf/aaai/IngmarBST20}.}
  \item[Representativity] $S$ should represent every alternative in
    $A$, i.e., each non-selected alternative from $A$ should be close
    to one from~$S$.
  \item[Efficiency]\footnote{Note that, similar to the social choice literature, we use the term \enquote{efficiency} as an umbrella term to refer to notions explicitly capturing solution quality. This differs from parts of the multiobjective literature, where efficiency is used as a synonym for Pareto optimality (see, e.g., \citep{ehrgott2005multicriteria})} $S$ should contain ``high-quality'' alternatives,
    i.e., alternatives which score well across objectives.
\end{description}
Which of these three desiderata is most important or appropriate
depends on the context and the demands of the DM, making it hard to
argue for or against each of them in general.
\subsection{Quality Measures}\label{sub:QM}
We focus on two of the arguably most popular quality measures\footnote{Technically speaking, our
    quality measures can also be viewed as objectives we optimize.
    However, to distinguish them from the objectives present in
    multiobjective optimization problems, we exclusively refer to them as
measures.} for Pareto pruning: uniformity and coverage
\citep{DBLP:journals/mp/Sayin00,DBLP:journals/candie/PetchrompoCBWP22,DBLP:journals/csur/LiY19}.
Inspired by the desiderata of diversity, the \emph{uniformity} of a slate $S$ is
$I_U(S) = \min_{x, y \in S} ||x-y|| = \min_{x,y \in S}\sum_{i=1}^d |
x_i - y_i|$.
\textsc{Uniformity Pareto Pruning} is the problem of finding a slate
$S$, i.e., a size-$k$ subset of $A$, maximizing uniformity
$ \max_{S \subseteq A, |S| = k} I_U(S)$.

Inspired by the  idea of representativity, the \emph{coverage} of a
slate $S$ with respect to a set of alternatives $A$ is
$ I_C(S, A) =  \max_{a \in A} \min_{s \in S} ||a - s|| = \max_{a \in
A} \min_{s \in S} \sum_{i = 1}^d |a_i - s_i|$.
Note that a lower coverage value is better, since it signals that
every point in $A$ is close to a point in $S$. \textsc{Coverage
Pareto Pruning} is the problem of finding a slate with a minimum coverage value
$ \min_{S \subseteq A, |S| =k} I_C(S, A)$.\footnote{Uniformity and
  coverage are connected. In \Cref{app:compare} we show that the
  optimal coverage value with $k$ points and the optimal uniformity
value with $k+1$ points differ by a factor of at most $2$.}

\paragraph{A New Quality Measure: Directed Coverage} Our new measure
\emph{directed coverage} is inspired by the coverage measure, but
aims to correct some of its flaws that surface in our axiomatic
analysis. The difference between the two is best illustrated by means
of the following example.
Consider $a = (1, 0)$ and $b = (0, \veps)$ for some small $\veps >
0$. Asked to present one alternative to the decision maker, which
alternative should we choose?
Coverage alone provides no guidance on which alternative is
preferable, yet there is a strong case that one should select option
$a$, since it significantly outperforms $b$ under objective one and
is almost as good as $b$ under objective two.
This is because coverage is based on the symmetric Manhattan
distance, making it irrelevant whether we take an efficient
alternative to cover a less-efficient one or the other way around.
Directed coverage fixes this issue:
When quantifying how suitable an alternative $s$ is to cover an
alternative $a$, we do not take into account the distance between the
two with respect to objectives in which $s$ outperforms $a$, as $s$
covers $a$ in these objectives ``perfectly'' in any case. Instead, we
purely focus on and sum over the objectives in which $a$ outperforms
$s$, i.e., $||a - s||_+$, as this quantifies the total efficiency
loss we suffer by presenting $s$ rather than $a$ to the decision maker.
Formally, we define the \emph{directed coverage} of a slate $S \subseteq A$ as
$
I_{DC}(S, A)  =   \max_{a \in A} \min_{s \in S} ||a - s||_+  =
\max_{a \in A} \min_{s \in S}\sum_{i = 1}^d \max(a_i - s_i, 0).
$
\textsc{Directed Coverage Pareto Pruning} is the problem of finding a
slate minimizing directed coverage:
$ \min_{s \in S, |S| = k} I_{DC}(S, A)$.

To illustrate the different selections made by the three measures, we
refer to \Cref{fig:comp} in the appendix, where we show their
behavior on instances from our experiments.

\subsection{Axiomatic Analysis} \label{sub:axAn}

\begin{table}
  \centering
    \begin{tabular}{l@{\hskip 2pt}c@{\hskip 4pt}c@{\hskip
      4pt}c@{\hskip 4pt}c@{\hskip 4pt}c}
      \toprule
      & \shortstack{Monotonicity}
      & \shortstack{$\veps$-Split\\Proofness}
      & \shortstack{Extremism\\Monotonicity}
      & \shortstack{Standout\\Consistency}
      & \shortstack{Outlier\\Consistency} \\
      \midrule
      Uniformity     & \xmark{} [Pr.\ \ref{prop:unic_not_monotonic}]
      & \cmark{} [Pr.\ \ref{prop:uni_epssplit}]     & \cmark{}
      [Pr.\ \ref{prop:uni_extreme}]         & \xmark{}
      [Pr.\ \ref{prop:unic_not_winner}]       & \xmark{}
      [Pr.\ \ref{prop:uni_dc_not_distance}] \\
      Coverage       & \xmark{} [Pr.\ \ref{prop:unic_not_monotonic}]
      & \xmark{} [Pr.\ \ref{prop:cdc_not_epssplit}] & \xmark{}
      [Pr.\ \ref{prop:cdc_not_extreme}]      & \xmark{}
      [Pr.\ \ref{prop:unic_not_winner}]       & \cmark{}
      [Pr.\ \ref{prop:c_distance}] \\
      Dir. Cov.  & \cmark{} [Pr.\ \ref{prop:dc_monotonic}] & \xmark{}
      [Pr.\ \ref{prop:cdc_not_epssplit}] & \xmark{}
      [Pr.\ \ref{prop:cdc_not_extreme}]      & \cmark{}
      [Pr.\ \ref{prop:dc_winner}]             & \xmark{}
      [Pr.\ \ref{prop:uni_dc_not_distance}] \\
      \bottomrule
  \end{tabular}
  \caption{Overview of axiomatic results. \cmark~indicates that the
  measure fulfills the axiom. \xmark~means that it violates it.}
  \label{tab:axioms}
\end{table}

While numerous quality measures have been proposed in the literature
\citep{DBLP:journals/csur/LiY19,faulkenberg2010quality}, there is a
lack of theoretical comparisons between them. In this section, we
conduct an axiomatic analysis of the three measures introduced above,
aiming to provide formal arguments for and against each measure. This
approach allows us to move beyond intuitive arguments for and against
different measures on disconnected grounds and instead evaluate
measures based on explicitly stated criteria.

We consider two types of axioms. The first type concerns how optimal
slates change in response to modifications of the underlying
instance. The second set examines whether certain “extreme”
alternatives are guaranteed to be included in an optimal slate. Our
axioms serve two main purposes: (i) to identify measures that exhibit
unintuitive or unreasonable behavior, and (ii) to identify how
measures align with the three desiderata introduced in \Cref{sec:PS}.
An overview of which measures satisfy which axioms is provided in
\Cref{tab:axioms}. Formal statements and proofs are given in \Cref{app:axan}.

We begin with the axiom of \emph{monotonicity}, which intuitively
demands that improving an alternative $x$ with respect to one or more
objectives should not result in $x$ being kicked out from the
selected slate. Such behavior would be counterintuitive, as it implies that strictly
improving an alternative's performance can make it less likely for
the DM to be presented with the alternative.
\begin{restatable}[Monotonicity]{axiom}{Monotonicity}\label{ax:monotonic}
  A measure $I$ satisfies \emph{monotonicity} if, for any set of
  alternatives $A$, $k\in \mathbb{N}$, and $S\in \mathcal{S}(I,A,k)$
  with $x\in S$, the following holds:
  If $y\in\mathbb{R}^d$ dominates $x$, then there exists an optimal slate
  \(S'\in \mathcal{S}\big(I,(A\setminus\{x\})\cup\{y\},k\big)\) with
  \(y\in S'\).
\end{restatable}
Both uniformity and coverage violate monotonicity. One reason for
this is that improving an alternative can reduce its Manhattan
distance to other alternatives, thereby diminishing its appeal to
diversity (as it decreases the quality measure) or coverage (as the
alternative becomes easier to cover). In contrast, directed coverage
avoids this issue: if $x$ strictly improves, then for any other
alternative $z$, $||x - z||_+$ can only increase, while $||z - x||_+$
can only decrease, implying that $x$ is not better covered by $z$ than before.
As a result, directed coverage satisfies monotonicity.

The second type of instance modification we consider is splitting an
alternative into two alternatives.
A popular variant of this idea, known as clone-robustness, requires
that adding a perfect duplicate of an alternative should not affect
the selected slate (up to potentially replacing the alternative with
the duplicate).
All three of our measures trivially satisfy clone-robustness, as
selecting two identical alternatives is never optimal.
To obtain a more meaningful distinction between measures, we consider
a stronger axiom, which we call \emph{$\varepsilon$-split proofness}.
It requires that no alternative $x$ can be replaced by two
arbitrarily close alternatives $y_\varepsilon$ and $z_\varepsilon$ so
that both $y_\varepsilon$ and $z_\varepsilon$ get selected.
Additionally, we demand that if either $y_\varepsilon$ or
$z_\varepsilon$ is selected in the modified instance, replacing them
with $x$ should still yield an optimal slate in the original
instance. This ensures that arbitrarily small perturbations cannot
cause any changes to the slate.

\begin{restatable}[$\veps$-split proofness]{axiom}{EpsSplitProof}
  A measure $I$ satisfies \emph{$\veps$-split proofness} if, for any
  set of alternatives $A$ and $k\in \mathbb{N}$, there exists some
  $\veps > 0$ such that for all $x \in A$ and $y_\veps, z_\veps \in
  \mathbb{R}^d$ with $||x - y_\veps|| < \veps$ and $||x - z_\veps|| <
  \veps$, the following holds:
  If $S_\veps\in \mathcal{S}\big(I,(A \setminus \{x\}) \cup
  \{y_\veps, z_\veps\},k\big)$, then
  \begin{enumerate*}[label=(\roman*)]
  \item $S_\veps \subseteq A$ and $S_\veps\in \mathcal{S}(I,A,k)$ or
  \item $S_\veps \setminus \{y_\veps, z_\veps\} \cup \{x\}\in
    \mathcal{S}(I,A,k)$.
  \end{enumerate*}
\end{restatable}
Notably, the axiom implies that a measure never selects two
alternatives that are arbitrarily close to one another, a property
particularly desirable from the perspective of the diversity
desideratum. Among the measures we consider, only uniformity
satisfies $\veps$-split proofness. Both coverage and directed
coverage violate the axiom, as it can be beneficial for these
measures to select two alternatives arbitrarily close to each other
if they cover different halves of the space.

While monotonicity and $\veps$-split proofness can be considered
broadly desirable, the desirability of the remaining axioms is more
subjective, as each of them captures some form of alignment with one
of the three desiderata introduced above.  We begin with a variant of
monotonicity tailored to the diversity desideratum, which we call
\emph{extremism monotonicity}. This axiom requires that if a selected
alternative is the most extreme according to some objective, then
pushing it even further away from the other alternatives in this
objective should not result in its exclusion from the slate.
\begin{restatable}[Extremism
  monotonicity]{axiom}{ExtremismMonotonicity}\label{ax:extreme}
  A measure $I$ satisfies \emph{extremism monotonicity} if for any
  set of alternatives $A$, $k\in\mathbb{N}$, $t>0$, and $S\in
  \mathcal{S}(I,A,k)$ with $x\in S$, the following holds:
  If for some objective $i\in[d]$, we have $x_i=\max_{a\in A} a_i$
  (resp.\ $x_i=\min_{a\in A} a_i$), then there exists an optimal slate
  \(S'\in \mathcal{S}\big(I,(A\setminus\{x\})\cup\{x'\},k\big)\) with
  \(x'\in S'\), where \(x'_i:=x_i+t\) (resp.\ \(x'_i:=x_i-t\)) and
  \(x'_j:=x_j\) for all \(j\in[d]\setminus\{i\}\).
\end{restatable}
This axiom formalizes the intuition that alternatives corresponding
to particularly distinct tradeoff decisions should remain part of the
slate when they become more distinct. As expected, uniformity
satisfies extremism monotonicity, while both coverage and directed
coverage violate~it.

Our next axiom is inspired by the notion of Condorcet-consistency.
Translated to our setting, Condorcet-consistency says that an
alternative outperforming each of the others in a majority of
objectives is always selected if it exists. We introduce a cardinal,
weighted variant based on a notion we call a \emph{standout alternative}.
To formalize this, we interpret $||x - y||_+$ as the ``lead'' of
alternative $x$ over alternative $y$, as it captures the total amount
by which $x$ outperforms $y$ across all objectives in which  $x$
outperforms $y$. An alternative is a standout alternative if its
weakest lead against any other alternative exceeds the strongest lead
any other alternative has against it:

\begin{restatable}[Standout
  consistency]{axiom}{StandoutConsistency}\label{ax:winner}
  An alternative $x \in A$ is a \emph{standout alternative} if
  $\min_{a \in A \setminus \{x\}} ||x-a||_+ \;>\; \max_{a \in A
  \setminus \{x\}} ||a-x||_+$.
  A measure $I$ is \emph{standout consistent} if for any set of
  alternatives $A$ containing a standout alternative $x\in A$, we
  have $x\in S$ for each optimal slate $S\in \mathcal{S}(I,A,k)$ and $k\ge 1$.
\end{restatable}

From the perspective of efficiency, standout alternatives are highly
desirable, as they are significantly better than all other
alternatives in aggregate. Among the measures we consider, only
directed coverage satisfies standout consistency. Uniformity and
coverage, in contrast, do not satisfy this axiom, as when faced with
the decision of which of two alternatives to pick, they do not take
into account which one is more efficient.

We conclude with the concept of an \emph{outlier alternative}, an
alternative that is further away from every other alternative than
any two non-outlier alternatives are from each other:

\begin{restatable}[Outlier
  consistency]{axiom}{OutlierConsistency}\label{ax:distance} An
  alternative $x \in A$ is an \emph{outlier alternative} if
  $\min_{a \in A \setminus \{x\}} \|x-a\| \;>\; \max_{y,z \in A
  \setminus \{x\}} \|y-z\|$.
  A measure $I$ is \emph{outlier consistent} if for any $k \geq 2$ and any set of
  alternatives $A$ containing an outlier alternative $x\in A$, we
  have $x\in S$ for each optimal slate $S\in \mathcal{S}(I,A,k)$.
\end{restatable}

From the perspective of representativity, an outlier should be
selected, as it lies too far from all other alternatives to be
adequately ``covered'' by any of them. Among the measures we
consider, only coverage satisfies outlier consistency, while both
uniformity and directed coverage do not.

\section{Algorithmic Analysis} \label{sec:alg}
We present our algorithmic analysis (see \Cref{tab:alg}). We start by
discussing some related problems from computational geometry
(\Cref{sec:CG}), before we analyze the complexity of Pareto pruning
for cardinal (\Cref{sec:CO}), ordinal (\Cref{sec:OO}), and approval
(\Cref{sec:AO}) objectives.

\begin{table}
  \centering
    \begin{tabular}{llccc}
      \toprule
      \multirow{2}{*}{Measure} & \#Objectives & Cardinal & Ordinal & Approval \\
      \cmidrule(l){2-5}
      & &       &       &       \\
      \midrule
      \multirow{3}{*}{Uniformity} & $d = 2$ & P$^\dagger$ &
      P$^\dagger$ & P$^\dagger$ \\
      & fixed $d \geq 3$ & NP-h [Th.\ \ref{thm:dim3_hard}] & ? & P
      [Pr.\ \ref{prop:lvalued_easy}] \\
      & unbounded $d$ & NP-h [Th.\ \ref{thm:dim3_hard}] & NP-h
      [Pr.\ \ref{thm:ord-dim3_hard}] & NP-h
      [Pr.\ \ref{prop:uc_approval_highd_hard}]\\
      \cmidrule(l){2-5}
      \multirow{3}{*}{Coverage} & $d = 2$ & P$^\dagger$ & P$^\dagger$
      & P$^\dagger$ \\
      & fixed $d \geq 3$ & NP-h [Th.\ \ref{thm:dim3_hard}] & ? & P
      [Pr.\ \ref{prop:lvalued_easy}] \\
      & unbounded $d$ & NP-h [Th.\ \ref{thm:dim3_hard}] & NP-h
      [Pr.\ \ref{thm:ord-dim3_hard}] & NP-h
      [Pr.\ \ref{prop:uc_approval_highd_hard}]\\
      \cmidrule(l){2-5}
      \multirow{4}{*}{Dir. Coverage} & $d = 2$ & P
      [Pr.\ \ref{prop:dc_r2_easy}]& P [Pr.\ \ref{prop:dc_r2_easy}] &
      P [Pr.\ \ref{prop:dc_r2_easy}] \\
      & fixed $d \geq 3$  & NP-h [Th.\ \ref{thm:dc_dim3_hard}] & ? &
      P [Pr.\ \ref{prop:lvalued_easy}] \\
      & unbounded $d$ & NP-h [Th.\ \ref{thm:dc_dim3_hard}] & NP-h
      [Pr.\ \ref{thm:ord-dim3_hard}] & NP-h
      [Pr.\ \ref{prop:uc_approval_highd_hard}] \\
      \bottomrule
    \end{tabular}
  \caption{Summary of computational results. Results marked with
    $\dagger$ are by
  \citet{DBLP:journals/cor/VazPFKS15}.}\label{tab:alg}
\end{table}

\subsection{Connections to Computational Geometry}\label{sec:CG}

\textsc{Uniformity Pareto Pruning} and \textsc{Coverage Pareto
Pruning} are special cases of geometric variants of two well-known
computational problems on graphs: the \textsc{Discrete $k$-Center}
problem \citep{hakimi_optimum_1964}  and the \textsc{$p$-Dispersion}
problem \citep{erkut_discrete_1990}. Given a set of points $B$, a
metric $d:B\times B \to \R_{\geq 0}$, and an integer $k$,
\textsc{Discrete $k$-Center}  (resp.\ \textsc{$p$-Dispersion}) asks
for a size-$k$ subset $S \subseteq B$ minimizing $\max_{a \in B}
\min_{s \in S} d(s, a)$ (resp.\ maximizing $\min_{x, y \in S, x \neq
y} d(x, y)$). Note that in case $d$ is the Manhattan distance, these
problems only differ from \textsc{Coverage Pareto Pruning}
(resp.\ \textsc{Uniformity Pareto Pruning}) in that $B$ and $S$ can
contain Pareto dominated points.
\citet{wangkuo1988} studied the geometric variant of
$\textsc{$p$-Dispersion}$ when $d$ is the Euclidean distance,
establishing NP-hardness in $\R^2$.
Considering the case when $d$ is the Euclidean or Manhattan distance,
\citet{megiddosupowit1984} showed NP-hardness in $\R^2$ for a
continuous version of \textsc{Discrete $k$-Center}, where the
selected points are not restricted to be from $B$, but one can select
any subset $S \subseteq \R^d$ of $k$ points.
In the literature, it is commonly assumed that \textsc{Discrete
$k$-Center} in $\R^2$ is NP-hard as well. However, we were unable to
track down a readily available proof.\footnote{For example:
  \citet{agarwal_efficient_1988} cite the works of
  \citet{megiddosupowit1984}, and \citet{fowler_optimal_1981} as a
  reference, yet both sources only contain a proof for the continuous
version.} To fill this gap and to use the results in our later
analysis, we provide a proof for the Manhattan distance in two
dimensions following the key ideas from \citet{megiddosupowit1984}:

\begin{restatable}{theorem}{kcentertwodhard}\label{thm:kcenter_hard}
  \textsc{Discrete $k$-center} for the Manhattan distance is NP-hard,
  even in two dimensions.
\end{restatable}

\subsection{Cardinal Objectives}\label{sec:CO}

When we restrict ourselves to Pareto optimal points in two dimensions, \textsc{Discrete $k$-Center} and \textsc{$p$-Dispersion} become tractable:
 \citet{DBLP:journals/cor/VazPFKS15} have presented polynomial-time
algorithms for \textsc{Uniformity Pareto Pruning} and
\textsc{Coverage Pareto Pruning} for the case of two objectives by
exploiting that a set of Pareto optimal alternatives $A \subseteq
\R^2$ can be embedded into $\R$ in a way that maintains the Manhattan
distance between alternatives. A dynamic programming approach for the
embedded problem in $\R$ yields a polynomial-time algorithm.
This general approach can also be adapted to showing an analogous
result for directed coverage:

\begin{restatable}{proposition}{dcdimtwoeasy}\label{prop:dc_r2_easy}
  For at most two objectives, \textsc{Directed Coverage Pareto
  Pruning} can be solved in $O(|A|k + |A| \log |A|)$.
\end{restatable}

\citet{DBLP:journals/cor/VazPFKS15} state in their conclusion:
``[Pareto pruning] for more than two objectives may become an
intractable task''. In fact, we were unable to find an NP-hardness
result for Pareto pruning, even for an arbitrary number of
objectives. We complement their tractability results with an
NP-hardness for \textsc{Uniformity / Coverage Pareto Pruning} for
three objectives.
We establish this result by adapting NP-hardness proofs for
\textsc{Discrete $k$-Center} and \textsc{$p$-Dispersion} for the
Manhattan distance in $\R^2$.
The general idea is that it is possible to construct a hyperplane $H
\subseteq \R^3$  in which there is no pair of points $x, y \in H$,
such that $x$ dominates $y$.
Embedding the constructions from these hardness proofs into such a
hyperplane $H$ then allows us to derive hardness results for
\textsc{Uniformity Pareto Pruning} and \textsc{Coverage Pareto
Pruning} for three objectives.

\begin{restatable}{theorem}{dimthreeprune}\label{thm:dim3_hard}
  \textsc{Uniformity / Coverage Pareto Pruning}  are NP-hard, even
  for three objectives.
\end{restatable}

The argument for \textsc{Directed Coverage Pareto Pruning} is more involved, as it is less clear how the directed distances change, when embedding a point set in $\R^2$ into $H$. Instead if we let $\Gamma$ be a triangular grid, writing $d_\Gamma$ for the distance metric induced by the grid, we show that there is an embedding $f$ of $\Gamma$ into $H$, such that $d_\Gamma(x, y) = ||f(x) - f(y)||_+$. The theorem then follows by adapting the proof for \textsc{Discrete $k$-Center} in two dimensions to a proof for \textsc{Discrete $k$-Center} on the triangular grid.

\begin{restatable}{theorem}{dcdimthreehard}\label{thm:dc_dim3_hard}
  \textsc{Directed Coverage Pareto Pruning} is NP-hard, even for three
  objectives.
\end{restatable}

\subsection{Ordinal Objectives} \label{sec:OO}
For the special case of ordinal objectives, the polynomial-time
algorithm for two objectives clearly still applies.
However, complementing this result with a hardness for a fixed number
of ordinal objectives turns out to be surprisingly difficult and
remains an open problem: The restriction of having to map bijectively
to $[|A|]$ is not strong enough to provide clear properties that an
algorithm can exploit, yet seems too restrictive to allow us to
nicely control the distances $||x - y||$ or $||x - y||_+$ within a
larger set of points.

Nevertheless, we show that all three problems are NP-hard for an
unbounded number of objectives.
For coverage and uniformity, the proof builds upon the hardness
proofs for \textsc{Discrete $k$-Center} and \textsc{$p$-dispersion}
in dimension two.
For directed coverage, we present a reduction from \textsc{Exact
Cover by 3-Sets}.

\begin{restatable}{proposition}{ucordinalhard}\label{thm:ord-dim3_hard}
  \textsc{Uniformity / Coverage /  Directed Coverage  Pareto Pruning}
  are NP-hard, even if all objectives are ordinal objectives.
\end{restatable}

\begin{table*}[t!]
  \resizebox{\textwidth}{!}{%
    \begin{tabular}{l ccc@{} c ccc@{} c ccc@{} c ccc@{} c ccc@{}}
      \toprule
      \textbf{Method} & \multicolumn{3}{c}{\textbf{Uniformity} $(\uparrow)$}
      & & \multicolumn{3}{c}{\textbf{Coverage} $(\downarrow)$}
      & & \multicolumn{3}{c}{\textbf{Directed Coverage} $(\downarrow)$}
      & & \multicolumn{3}{c}{\textbf{Hypervolume} $(\uparrow)$}
      & & \multicolumn{3}{c}{\textbf{Avg. Sum Objective} $(\uparrow)$} \\
      \cmidrule{2-4} \cmidrule{6-8} \cmidrule{10-12} \cmidrule{14-16}
      \cmidrule{18-20}
      & $k=5$\% & $k=10$\% & $k=25$\%
      & & $k=5$\% & $k=10$\% & $k=25$\%
      & & $k=5$\% & $k=10$\% & $k=25$\%
      & & $k=5$\% & $k=10$\% & $k=25$\%
      & & $k=5$\% & $k=10$\% & $k=25$\% \\
      \midrule
      \multicolumn{19}{l}{\textbf{\textit{Dataset ZDT}}} \\
      \texttt{Uniformity}     & 100.0\% & 100.0\% & 100.0\% &&
      121.8\% & 108.3\% & 117.4\% && 195.3\% & 187.1\% & 150.4\% &&
      90.8\% & 97.9\% & 99.9\% && 93.5\% & 96.9\% & 98.1\% \\
      \texttt{Coverage}       & 81.0\%  & 83.0\%  & 74.7\%  &&
      100.0\% & 100.0\% & 100.0\% && 201.0\% & 204.1\% & 182.3\% &&
      97.3\% & 98.9\% & 99.9\% && 94.3\% & 97.0\% & 98.4\% \\
      \texttt{Dir. Coverage}  & 69.3\%  & 68.4\%  & 41.3\%  &&
      165.7\% & 212.9\% & 364.6\% && 100.0\% & 100.0\% & 100.0\% &&
      99.3\% & 99.8\% & 99.9\% && 99.1\% & 99.5\% & 99.6\% \\

      \midrule
      \multicolumn{19}{l}{\textbf{\textit{Dataset DTLZ}}} \\
      \texttt{Uniformity}     & 100.0\% & 100.0\% & 100.0\% &&
      131.2\% & 125.1\% & 123.2\% && 200.0\% & 158.1\% & 158.5\% &&
      92.6\% & 98.6\% & 99.2\% && 98.1\% & 98.0\% & 99.5\% \\
      \texttt{Coverage}       & 70.3\%  & 72.8\%  & 66.7\%  &&
      100.0\% & 100.0\% & 100.0\% && 178.2\% & 188.7\% & 185.6\% &&
      99.9\% & 97.9\% & 95.9\% && 94.8\% & 97.2\% & 97.1\% \\
      \texttt{Dir. Coverage}  & 72.2\%  & 58.9\%  & 59.1\%  &&
      155.8\% & 188.7\% & 246.6\% && 100.0\% & 100.0\% & 100.0\% &&
      96.4\% & 99.6\% & 99.1\% && 98.5\% & 97.9\% & 97.1\% \\

      \midrule
      \multicolumn{19}{l}{\textbf{\textit{Dataset PGMORL}}} \\
      \texttt{Uniformity}     & 100.0\% & 100.0\% & 100.0\% &&
      123.0\% & 132.0\% & 144.8\% && 187.5\% & 234.6\% & 254.7\% &&
      94.8\% & 97.9\% & 99.5\% && 96.0\% & 97.2\% & 98.7\% \\
      \texttt{Coverage}       & 79.9\%  & 60.8\%  & 56.4\%  &&
      100.0\% & 100.0\% & 100.0\% && 169.7\% & 182.3\% & 227.1\% &&
      98.3\% & 98.7\% & 99.4\% && 96.9\% & 97.6\% & 98.7\% \\
      \texttt{Dir. Coverage}  & 51.1\%  & 38.5\%  & 45.9\%  &&
      280.4\% & 347.4\% & 482.3\% && 100.0\% & 100.0\% & 100.0\% &&
      100.0\% & 100.0\% & 100.0\% && 100.0\% & 100.0\% & 100.0\% \\

      \bottomrule
    \end{tabular}
  }

  \caption{Comparison of three methods for Pareto pruning. We report
    average values of five measures, each normalized by the best
    solution at the instance level. $(\uparrow)$ indicates that higher
    values are better, and $(\downarrow)$ indicates that lower values
  are better.} \label{table:main_res}
\end{table*}

\subsection{Approval Objectives}\label{sec:AO}

For approval objectives, our problems become easier from a
computational perspective. We establish polynomial-time solvability
for every fixed number of objectives $d \in \N$: For this, we call
two alternatives \emph{equivalent} if they are evaluated the same
under every objective.
Observe that there can be at most $2^d$ pairwise non-equivalent alternatives.
As it is never optimal for any of our measures to pick two equivalent
alternatives, it suffices to brute force over all at most
$\binom{2^d}{k} \leq 2^{2^d}$ size-$k$ subsets of pairwise
non-equivalent alternatives:

\begin{restatable}{proposition}{lvaluedeasy}\label{prop:lvalued_easy}
  For any fixed $d\in \mathbb{N}$, \textsc{Uniformity / Coverage /
  Directed Coverage Pareto Pruning} are solvable in polynomial time
  for $d$ approval objectives.\footnote{The result extends to all
  $l$-valued objectives for fixed $l\in \mathbb{N}$.}
\end{restatable}

We complement this result with an NP-hardness result for an unbounded
number of approval objectives.
To show this result, we draw inspiration from the classic
(non-metric) hardness proofs for \textsc{Discrete $k$-center} and
\textsc{$p$-dispersion} on graphs \citep{hakimi_optimum_1964,
erkut_discrete_1990}.
The idea is that given a graph $G = (V, E)$, we construct an
alternative $a_v \in A$ for every $v \in V$ and add objectives such
that the distance between $a_v$ and $a_w$ is small if $\{a_v, a_w\}
\in E$ and large otherwise. Hardness is then a straightforward
reduction from \textsc{Independent set} for uniformity, and
\textsc{Dominating set} for coverage and directed coverage.

\begin{restatable}{proposition}{ucapprovalhard}\label{prop:uc_approval_highd_hard}
  \textsc{Uniformity / Coverage / Directed Coverage Pareto Pruning}
  are NP-hard, even if all objectives are approval objectives.
\end{restatable}

\section{Experiments} \label{sec:exp}

We conduct an experimental evaluation of the slates returned by the
three solution methods we consider. In this section, we use the terms
uniformity, coverage, and directed coverage to refer both to the
underlying quality measures ($I_U$, $I_C$, and $I_{DC}$,
respectively), which we use to evaluate slates, and the respective
solution methods that optimize for one of them. To distinguish, we
use typewriter font when referring to the solution method, i.e., the
slate obtained by solving the corresponding optimization problem
(e.g., we write \texttt{Uniformity} to refer to  \textsc{Uniformity
Pareto Pruning}).

\paragraph{Setup}
We consider three different datasets. Datasets ZDT \citep{zdt_2000} containing six
instances with two objectives and DTLZ \citep{dtlz_2002} containing seven instances with three objectives
 are widely used for the evaluation of multiobjective
evolutionary algorithms.\footnote{The Pareto fronts of these problems
are taken from the pymoo \citep{pymoo_2020} library.}
For a more realistic example, we consider the dataset PGMORL
containing six instances, where the alternatives correspond to
simulated agents performing a simple task evaluated under two
objectives. \citet{pgmorl_2020} created these benchmark instances to
evaluate their multiobjective evolutionary algorithm
PGMORL.\footnote{We use the Pareto fronts calculated by \citet{pgmorl_2020}.}
We compute all slates via integer linear programming (ILP)
formulations, solved using Gurobi.
For feasibility reasons, for the six instances from these datasets in
which the Pareto front contains more than $200$ alternatives, we
delete all but $200$ randomly sampled alternatives from the instance.
We consider three different values of $k$, i.e., $k=5\%\cdot |A|$,
$k=10\%\cdot |A|$, and $k=25\%\cdot |A|$.

\paragraph{Results}
We evaluate each computed slate $S$ using five quality measures:
uniformity $I_U$, coverage $I_C$, directed coverage $I_{DC}$,
hypervolume,\footnote{For a set of alternatives $A$, and a reference
  point $r$, the hypervolume of $S \subseteq A$ is the volume of $C =
  \{x \in \R^d \mid x \text{ dominates } r \text{ and } x \text{ is
  dominated by some } a \in S\}$.
  Hypervolume is seen to capture both efficiency and diversity of
alternatives.} and the average summed quality of the selected
alternatives, i.e., $\nicefrac{1}{k} \sum_{a \in S} \sum_{j=1}^d
a_j$. The last two measures capture different notions of slate
efficiency. To enable a meaningful comparison across solution methods
and aggregation across instances, we normalize all scores within each
instance by dividing by the score of the best-performing slate under
the respective measure. For example, when evaluating uniformity
$I_U$, we divide the uniformity score of each slate by the maximum
uniformity achieved across all methods, which is by definition
\texttt{Uniformity}, for that instance. \Cref{table:main_res} reports
the normalized values, averaged over all instances in each dataset.
For measures marked with $(\uparrow)$, higher values indicate better
performance; for those marked with $(\downarrow)$, lower values are~preferred.

\paragraph{Analysis} We discuss some patterns observed in
\Cref{table:main_res}. While the choice of $k$ does influence
methods' performance, no consistent influence of changing $k$ is
visible. Therefore, we focus on observations that hold across all
three considered values of $k$.
First, we observe substantial relative differences between the three
solution methods in terms of their performance under the uniformity,
coverage, and directed coverage measures. This underscores that the
choice of method can have significant practical implications. We
observe that \texttt{Coverage} consistently outperforms
\texttt{Directed Coverage} with respect to uniformity, and
\texttt{Uniformity} outperforms \texttt{Directed Coverage} with
respect to coverage. This suggests that, despite differences in their
formal definitions, \texttt{Uniformity} and \texttt{Coverage} exhibit
more similar behavior to each other than either does to
\texttt{Directed Coverage}. In contrast, when evaluating performance
under the directed coverage measure, no consistent trend emerges as
to whether \texttt{Uniformity} or \texttt{Coverage} performs better.
However, both return slates that, from the perspective of directed
coverage, are typically more than $50\%$ worse than those produced by
the dedicated \texttt{Directed Coverage} method. This illustrates
that if one cares about the directed coverage measure, using one of
the two more established approaches is insufficient.

When evaluating performance with respect to hypervolume and average
summed objective value, which are more efficiency-focused, the
differences between the solution methods are less pronounced. On  ZDT
and PGMORL, \texttt{Directed Coverage} consistently outperforms
\texttt{Coverage}, which in turn outperforms \texttt{Uniformity}. For
DTLZ, which method performs better depends on the choice of~$k$.
While the differences are smaller than for the other measures, these
results still provide evidence that \texttt{Directed}
\texttt{Coverage} tends to select more efficient solutions. This is
also intuitive: by design, \texttt{Directed} \texttt{Coverage} avoids
selecting alternatives that are only marginally better in some
objectives while being worse in all others in comparison to other
alternatives. At the instance level, we further observe that, unlike
the other two methods, \texttt{Directed} \texttt{Coverage} tends to
avoid selecting large numbers of alternatives from regions populated
by less-efficient alternatives; see \Cref{fig:comp} in the appendix
for some examples.

In \Cref{app:experiments} and our supplementary material available on github, we include further plots that support and
extend our findings. For example, on the instance level, we observe
that for all methods, increasing $k$  yields substantial improvements
in the quality measures when $k$ is small. However, as $k$ grows, the
marginal gains diminish considerably. We also find that the
performance of a solution method with respect to a measure it does
not explicitly optimize can vary significantly with small changes~in~$k$.

\section{Discussion} \label{sec:disc}
We presented a systematic study of quality measures for  Pareto
pruning, including the first axiomatic analysis and a comprehensive
complexity investigation. We hope that our work enables more
principled arguments for and against different measures in
multiobjective optimization and contributes to a clearer
understanding of their tractability. Motivated by the shortcomings
revealed in our axiomatic analysis, we proposed the new measure of
directed coverage, which performs competitively or even favorably in
our experiments.

There are several promising directions for future work. First, it
would be valuable to complement our axiomatic analysis with
characterization and impossibility results, and design axioms tailored to ordinal or approval objectives (in particular, $\varepsilon$-split proofness and extremism monotonicity do not translate to these settings).
Second, our algorithmic analysis leaves open whether Pareto pruning remains hard for ordinal objectives with a fixed number of objectives.
Third, extending our analysis to further quality measures would be worthwhile.
Lastly, it would be intriguing to further explore the connection
between Pareto pruning and previous work in social choice,
particularly the paradigms of proportional representation and
diversity in multiwinner voting \citep{faliszewski2017multiwinner}.
While in our work, we interpreted solutions as candidates and
objectives as voters, it would also be fruitful to explore a social
choice modeling in which solutions serve as both candidates and
voters, ranking other solutions by similarity.
This would embed the problem in recent work on centroid clustering in
the social choice literature
\citep{DBLP:conf/icalp/Micha020,DBLP:conf/nips/Kellerhals024}. It
would be interesting to analyze whether existing Pareto pruning
methods satisfy solution concepts from this setting, and conversely,
whether algorithms from that literature can offer meaningful
guarantees or performances for Pareto pruning.

\newpage
\clearpage

\setcounter{secnumdepth}{2}
\renewcommand{\thesubsection}{\thesection.\arabic{subsection}}

\appendix
\newpage{}
\onecolumn

\section{Background on Multiobjective Optimization} \label{sec:RW}

In the multiobjective optimization literature, various paradigms have
been developed to incorporate the DM into the solution process (see
  the surveys by
\citet{marler2004survey,DBLP:journals/candie/PetchrompoCBWP22,DBLP:conf/dagstuhl/2008moo,ehrgott2005multicriteria}).
We refer to the classification by
\citet{DBLP:journals/candie/PetchrompoCBWP22}, which distinguishes
between four approaches:: \emph{a priori}, \emph{interactive},
\emph{a posteriori}, and \emph{pruning}. A priori methods assume that
the DM’s complete preferences over objectives are available before
the optimization and collapse the multiobjective problem into a
(weighted) single-objective one \citep{marler2010weighted}.
Interactive methods iteratively improve the solution by alternating
between eliciting feedback from the DM on specific solutions and
updating the solution accordingly
\citep{DBLP:journals/access/XinCCIHL18}. A posteriori methods aim to
generate a large set of Pareto optimal solutions for the DM to choose
from \citep{marler2004survey}, whereas pruning methods seek to reduce
this set by selecting a smaller, representative subset to avoid
overwhelming the DM \citep{DBLP:journals/ress/TaboadaBCW07}.

The approach discussed in this paper belongs to the category of
\emph{post-optimality} pruning methods, where pruning occurs after
computing the Pareto front; in contrast,  \emph{intra-optimality}
methods integrate pruning directly into the optimization algorithm
\citep{DBLP:conf/aaai/DemirovicS20,DBLP:conf/emo/EmmerichBN05}.
In the absence of additional information of the DM's preferences,
post-optimality pruning methods typically either aim to put together
a set of solutions with a good total performance
\citep{DBLP:conf/ppsn/BrankeDDO04} or to select a set of solutions
that reflect the entirety of the Pareto front  \citep{taboada2007data}.
A common strategy for achieving these goals, also pursued in this
paper, is to maximize a predefined quality or diversity measure.
Alternative approaches include applying clustering algorithms to
group similar solutions
\citep{DBLP:journals/eor/ZioB11,taboada2007data} and manually
selecting well-distributed solutions from the Pareto front
\citep{DBLP:journals/ress/WangLFLC20}.

\section{Axiomatic Analysis}\label{app:axan}

\subsection{Monotonicity}

\Monotonicity*

\begin{proposition}\label{prop:dc_monotonic}
  Directed Coverage satisfies monotonicity.
\end{proposition}
\begin{proof}
  Let $A$ be the set of alternatives, $x \in A$, $y \in \R^d$ and $x$
  dominated by  $y$.
  For any $z \in A$, we get $||z - y||_+ \leq ||z - x||_+$ and $||x -
  z ||_+ \leq ||y - z||_+$, since $z_i - y_i \leq z_i - x_i$ and $x_i
  - z_i \leq y_i - z_i$ for all $i \in [d]$.

  Let $S \in \mathcal{S}(I_{DC}, A, k)$ with $x \in S$, $A_y \coloneq
  A \setminus \{x\} \cup \{y\}$ and $S' \subseteq A \setminus \{x\}$.
  We get $I_{DC}(S', A_y) \geq I_{DC}(S', A) \geq I_{DC}(S, A) \geq I_{DC}(S
  \setminus\{x\}\cup\{y\}, A_y)$.
  Therefore any slate not containing $y$ is at most as good as $S
  \setminus\{x\}\cup\{y\}$ in $A_y$, so any $S_y \in
  \mathcal{S}(I_{DC}, A_y, k)$ must contain $y$.
\end{proof}

\begin{proposition}\label{prop:unic_not_monotonic}
  Uniformity and Coverage do not satisfy monotonicity.
\end{proposition}
\begin{proof}
  For uniformity, let $A = \{(2, 0, 0), (0, 2, 0), (0, -2, 1)\}$ and $k = 2$.
  By checking all slates with $2$ elements, one can verify that an
  optimal solution to the uniformity problem is $S = \{(2, 0, 0), (0, -2, 1)\}$.
  Replacing  $(0, -2, 1)$ by $(0, 0, 1)$ we see that the only optimal
  slate for $A' = \{(2, 0, 0), (0, 2, 0), (0, 0, 1)\}$ is $S = \{(0,
  2, 0), (2, 0, 0)\}$.

  For coverage consider $A = \{(3, -10, 0), (1, 3, 0), (2, 2, 1), (0,
  0, 3)\}$. The optimal slate for coverage with $k = 2$ is given by
  $S = \{(3, -10, 0), (2, 2, 1)\}$ for a coverage value of $6$.
  Now replacing  $(3, -10, 0)$ by $(3, 1, 0)$, we get $A' = \{(3, 1,
  0), (2, 2, 1), (1, 3, 0),  (0, 0, 3)\}$. The optimal solution for
  coverage is $S' = \{(2, 2, 1), (0, 0, 3)\}$ with a coverage value
  of $I_C(S', A') = 3$.
\end{proof}

\subsection{$\veps$-Split Proofness}\label{ax:epssplit}
\EpsSplitProof*

\begin{proposition}\label{prop:uni_epssplit}
  Uniformity is $\veps$-split proof.
\end{proposition}
\begin{proof}
  Let $S \in \mathcal{S}(I_U, A, k)$. Let $D = \{||x - y|| \mid x,y
  \in A, x \neq y\}$ be the set of distances occurring between
  alternatives of $A$ and $d_{\min} = \min_{d_1, d_2 \in D, d_1 \neq d_2} |d_1
  -d_2|$ be the minimal difference between two such distances.
  Let $0 < \veps < \min\left(\frac{I_U(S)}{3}, \frac{d_{\min}}{2}\right)$.
  Lastly, let $x \in A$ and $y_\veps, z_\veps \in \R^d$ with $||x -
  y_\veps|| < \veps$ and $||x - z_\veps|| < \veps$.

  Let $A_\veps \coloneq A \cup \{y_\veps, z_\veps\} \setminus \{x\}$.
  First note that for any $a \in A$ we get $||a - y_\veps|| \geq ||a
  - x|| - ||x - y_\veps|| > ||a - x|| - \veps$ due to the triangle inequality.
  In turn, it also follows for any $T \subseteq A$ with $x \in T$, that $I_U(T) >
  I_U(T \setminus \{x\} \cup \{y_\veps\}) - \veps$.

  Let $S_\veps \in \mathcal{S}(I_U, A_\veps, k)$. $S_\veps$ cannot
  contain both $y_\veps$ and $z_\veps$: We have $I_U(S \setminus \{x\}
  \cup \{y_\veps\}) > I_U(S) - \veps > \frac{2}{3} I_U(S) > 2\veps$,
  since $\veps < \frac{I_U(S)}{3}$. Therefore, any optimal $S_\veps
  \subseteq A_\veps$ cannot contain both $y_\veps$ and $z_\veps$,
  since $||y_\veps - z_\veps|| \leq ||y_\veps - x|| + ||x - z_\veps||
  < 2 \veps$. $S \setminus \{x\} \cup \{y_\veps\}$ would achieve a higher
  uniformity score in $A_\veps$.

  Without loss of generality suppose $y_\veps \in S_\veps$ and
  $z_\veps \notin S_\veps$.  Define $S' \coloneq S \setminus \{x\}
  \cup \{y_\veps\}$ and $S_\veps' \coloneq S_\veps \setminus
  \{y_\veps\} \cup \{x\}$. Then we use the observation above and that
  $S_\veps$ is optimal to derive
  \[I_U(S) - 2\veps < I_U(S') - \veps \leq I_U(S_\veps) - \veps <
  I_U(S_\veps') \leq I_U(S)\]  and therefore $|I_U(S) -
  I_U(S_\veps')| < 2\veps < d_{\min}$. Now, since $I_U(S)$ and
  $I_U(S_\veps')$ are both members of $D$, they must be equal, since
  their difference is smaller than the minimum difference of two
  elements in $D$. Therefore $S_\veps'$ is an optimal slate.

  Lastly suppose $y_\veps, z_\veps \notin S_\veps$. As $S$ is an
  optimal slate of $A$ and $S_\veps \subseteq A$, $I_U(S_\veps) \leq I_U(S)$.
  If further $x \notin S$, then also $I_U(S) \leq I_U(S_ \veps)$,
  since $S_\veps$ is optimal and $S, S_\veps \subseteq A_\veps$.
  If $x \in S$, then $I_U(S) \leq I_U(S') + \veps \leq I_U(S_\veps) +
  \veps$. Either way $|I_U(S) - I_U(S_\veps)| < \veps < d_{\min}$ and
  therefore $I_U(S) = I_U(S_\veps)$. So $S_\veps$ is optimal in $A$.
\end{proof}

\begin{proposition}\label{prop:cdc_not_epssplit}
  Coverage and Directed Coverage are not $\veps$-split proof.
\end{proposition}
\begin{proof}
  Consider the set $A = \{(-1, 1),(-1, -1),(0, 0),(1, -1),(1, 1)\}$.
  One can then check that
  \[\min_{S \subseteq A, |S| = 2} \max_{a \in A} \min_{s \in S} ||a - s|| = 2.\]
  This minimum is achieved at any $S$ containing $(0, 0)$, for example $S = \{(0, 0), (1, 1)\}$. Now replace
  $(0, 0)$ by $(-\veps, 0), (\veps, 0)$.
  So, let $A_\veps = \{(-1, 1),(-1, -1),(-\veps, 0),(\veps, 0), (1,
  -1),(1, 1)\}$
  \[\min_{S \subseteq A_\veps, |S| = 2} \max_{a \in A} \min_{s \in S} ||a - s|| = 2-\veps.\]
  However, this minimum is uniquely achieved at $S_\veps = \{(-\veps, 0),
  (\veps, 0)\}$.

  To construct a counterexample for Coverage and directed Coverage
  consider the map $f: \R^2 \to \R^4, (x_1, x_2) \mapsto
  (\frac{x_1}{2},\frac{-x_1}{2},\frac{x_2}{2},\frac{-x_2}{2})$.
  One can verify that $f$ is injective and for $x,y \in \R^2$ it
  holds that $||f(x) - f(y)|| = ||x - y||$ and  $||f(x) - f(y)||_+ =
  \frac{||x-y||}{2}$. More so, for any $x,y \in \R^2$, $f(x)$ does
  not dominate $f(y)$ and $f(y)$ does not dominate $f(x)$.
  Therefore, $f(A)$ and $f(A_\veps)$ constitute valid counterexamples for
  coverage and directed coverage.
\end{proof}

\subsection{Extremism Monotonicity}

\ExtremismMonotonicity*

\begin{proposition}\label{prop:uni_extreme}
  Uniformity satisfies extremism monotonicity.
\end{proposition}
\begin{proof}
  For some $i \in [d]$ let $x_i \in \min_{a \in A} a_i$ and $A' = A
  \setminus \{x\} \cup \{x'\}$, then for any $y \in A$ we get
  $||x' - y|| = t + ||x - y||$, since $x_i$ was minimal. Let $S'
  \subseteq A \setminus \{x\}$, with $|S'| = k$ and $S \in
  \mathcal{S}(I_U, A, k)$ with $x \in S$ , then we get $I_U(S) \geq
  I_U(S')$, due to the optimality of $S$ and $I_U(S \setminus \{x\}
  \cup \{x'\}) \geq I_U(S)$ due to the observation above.
  Therefore $I_U(S') \leq I_U(S) \leq I_U(S \setminus \{x\} \cup \{x'\})$. This
  implies that either $S \setminus \{x\} \cup \{x'\}$ is already
  optimal, or any slate achieving a higher uniformity score must contain $x'$.
  The case where $x_i \in \max_{a \in A} a_i$ can be handled by an
  analogous argument.
\end{proof}

\begin{proposition}\label{prop:cdc_not_extreme}
  Coverage and directed Coverage do not satisfy extremism monotonicity.
\end{proposition}
\begin{proof}
  Let $k = 1$ and $A = \{(3, 0, 0), (0, 3, 0), (2, 1, 1)\}$. The only
  optimal slate for coverage is $S = \{(2, 1, 1)\}$, with $I_C(S, A) = 5$.
  In particular $(2, 1, 1)$ is an extreme alternative for objective
  three. However, the optimal slate in $A' = \{(3, 0, 0), (0, 3, 0),
  (2, 1, 3)\}$ is
  $S' = \{(3, 0, 0)\}$ for a coverage value of $I_C(S', A') = 6$.

  For directed coverage consider $k = 1$ and $A = \{(2, 0), (0,
  1)\}$, then $\{(2, 0)\}$ is the unique optimal slate for
  directed coverage. In particular, $(2, 0)$ is minimal under the second objective. But in $A' = \{(2, -3),
  (0, 1)\}$ the only
  optimal slate is $\{(0, 1)\}$.
\end{proof}

\subsection{Standout Consistency}
\StandoutConsistency*
\begin{proposition}\label{prop:dc_winner}
  Directed coverage is standout consistent.
\end{proposition}

\begin{proof}
  Let $x \in A$ be a standout alternative. Let $l = \min_{a \in A
  \setminus \{x\}} ||x - a||_+$ and $r = \max_{a \in A \setminus
  \{x\}}||a - x||_+$. Let $S \subseteq A \setminus \{x\}$, then
  \[I_{DC}(S,A) = \max_{a \in A} \min_{s \in S} ||a - s||_+ \geq
    \min_{s \in S} ||x - s||_+ \geq \min_{a \in A \setminus \{x\}} ||x
  - a||_+ = l.\]
  On the other hand, for $S' \subseteq A$ with $x \in S$, we get
  \[I_{DC}(S', A) = \max_{a \in A} \min_{s \in S'} ||a - s||_+ \leq
  \max_{a \in A} ||a - x||_+ = \max_{a \in A \setminus \{x\}} ||a - x||_+ = r\]
  Since $x$ is a standout alternative we get \[I_{DC}(S, A) \geq l > r
  \geq I_{DC}(S', A),\]
  so any slate minimizing $I_{DC}$ must contain $x$.
\end{proof}
\begin{proposition}\label{prop:unic_not_winner}
  Uniformity and Coverage are not standout consistent.
\end{proposition}
\begin{proof}
  Let $A = \{(0, 1), (2, 0)\}$ and $k = 1$, then for both uniformity
  and coverage $S = \{(0, 1)\}$ is an optimal slate, but $(2, 0)$ is
  the unique standout alternative in $A$.
\end{proof}

\subsection{Outlier Consistency}
\OutlierConsistency*
\begin{proposition}\label{prop:c_distance}
  Coverage is outlier consistent.
\end{proposition}

\begin{proof}
  Let $x \in A$ be an outlier alternative. Let $l = \min_{a \in
  A\setminus\{x\}} ||x - a||$ and $r = \max_{y, z \in A \setminus
  \{x\}}||y-z||$. Let $S \subseteq A \setminus \{x\}$, with $|S| \geq 2$, then
  \[I_{C}(S, A) = \max_{a \in A} \min_{s \in S} ||a - s|| \geq \min_{s
  \in S} ||x - s|| \geq \min_{a \in A \setminus \{x\}} || x - a|| = l.\]
  On the other hand, for $S' \subseteq A$ with $x \in S'$ and some
  other $b \in S' \setminus \{x\}$, we get
  \[I_{C}(S', A) = \max_{a \in A} \min_{s \in S'} ||a - s|| \leq \max_{a
    \in A \setminus \{x\}} ||a - b|| \leq \max_{y,z \in A \setminus
  \{x\}} ||y-z|| = r\]
  Since $x$ is an outlier alternative we get \[I_{C}(S) \geq l > r
  \geq I_{C}({S'}),\]
  so any slate minimizing $I_{C}$ must contain $x$.
\end{proof}
\begin{proposition}\label{prop:uni_dc_not_distance}
  Uniformity and Directed Coverage are not outlier consistent.
\end{proposition}
\begin{proof}
  Let $A = \{(1, 0, 0,0), (0, 1, 0,0), (0, 0, 1,0), (0,0,0,20)\}$ and
  $k = 3$, then for uniformity, any optimal slate achieves at most a
  uniformity value of $2$. In fact, $\{(1, 0, 0, 0), (0, 1, 0, 0), (0,
  0, 1, 0)\}$ is optimal, but $(0, 0, 0, 20)$ is a outlier alternative.

  For directed coverage consider $A = \{(1, 0, 0, 0), (0, 1, 0, 0),
  (0, 0, 1, -10)\}$. Then $(0, 0, 1, -10)$ is an outlier, but $\{(1,
  0, 0, 0), (0, 1, 0, 0)\}$ is an optimal slate.
\end{proof}

\section{A Relation Between Metric $k$-center and Metric
$p$-dispersion}\label{app:compare}

We introduce the following notation for the $k$-center and
$p$-dispersion problems. For a metric space $(X, d)$ and a subset $S
\subseteq X$ we define $C(S, X) = \max_{x \in X}\min_{s \in S}d(x,
s)$ and $U(S) = \min_{x,y \in S} d(x, y)$.

\citet{shier_min-max_1977} noted a duality between $p$-dispersion
$k$-center on tree graphs: An optimal $k$-center solution with
$k$-points achieves the same objective value as an optimal
$p$-dispersion solution with $k+1$-points. While this duality does
not yield an equality in our setting, we can still prove that the
objective values only differ by a factor of at most $2$.

\begin{proposition}
  Let $(X, d)$ be a metric space and $k \in \N, k \geq 1$. Let $K_k =
  \min_{S \subseteq X, |S| = k}C(S, X)$ and $M_{k+1} = \max_{S
  \subseteq X, |S| = k+1} U(S)$. Then $K_k \leq M_{k+1} \leq 2 K_k$.
\end{proposition}
\begin{proof}
  To see that $M_{k+1} \leq 2 K_k$, we turn to the original proof of
  \citet{shier_min-max_1977}: Let $S_k \subseteq X, |S_k| = k$ with
  $C(S_k, X) = K_k$. So for every point $x \in X$ there exists a
  point $s \in S_k$ such that $d(s, x) \leq K_k$. Thus for any set $T
  \subseteq X$ with $|T| = k+1$, by the pigeon hole principle there
  must exist distinct $a, b \in T$ and a $s \in S$ such that $d(a, s)
  \leq K_k$ and $d(b, s) \leq K_k$. Now by the triangle inequality
  $d(a, b) \leq 2 \cdot K_k$. Since this holds for any set $T$ with
  $k+1$ points, $M_{k+1} \leq 2 K_k$ follows.

  To show $K_k \leq M_{k+1}$, let $T_{k+1} \subseteq X$ be a subset
  lexicographically maximizing the sorted sequence $(d(u, v))_{u, v
  \in T_{k+1}}$. Since $T_{k+1}$ is a lexicographic maximizer, it
  also maximizes $\min_{x,y \in T_{k+1}} d(x,y) = U(T_{k+1})$ and
  therefore $U(T_{k+1}) = M_{k+1}$. Let $a, b \in T_{k+1}$ with $d(a,
  b) = M_{k+1}$. We now claim that $C(T_{k+1} \setminus \{b\}, X)
  \leq M_{k+1}$. Suppose there is a point $x \in X$ such that $d(x,
  t) > M_{k+1}$ for all $t \in T_{k+1} \setminus \{b\}$, then
  $T_{k+1} \setminus \{b\} \cup \{x\}$ is lexicographically larger
  than $T_{k+1}$, since we have replaced an occurrence of the minimal
  distance $M_{k+1}$ in the sorted sequence $(d(u,v))_{u,v \in
  T_{k+1}}$ by only strictly larger distances involving
  $x$. Therefore for some optimal $S$, with $|S| = k$ we get $K_k =
  C(S, X) \leq C(T_{k+1} \setminus \{b\}, X) \leq M_{k+1}$.
\end{proof}

\section{Algorithmic Analysis}
\subsubsection{}

\kcentertwodhard*

\begin{proof}
  We show NP-hardness via a reduction from \textsc{3SAT}. Let $Q =
  (V, \mathcal{W})$ be a SAT formula, where $V$ is the set of
  variables and $\mathcal{W}$ is the set of clauses.
  For a given SAT formula we aim to construct a set of points
  $A \subseteq \R^2$ such that a $k$-center solution $S$ with a
  certain value exists if and only if the SAT formula admits a solution.
  The basic idea is to encode variables $i \in V$ as large
  \emph{circuits} $C_i$ of points in $\R^2$, with each of the two assignments for $i$
  encoded as one of two possible options for $C_i \cap S$ for any set $S$ that
  achieves a certain threshold under the $k$-center measure.
  To represent a \emph{clause} $W \in \mathcal{W}$, one then
  geometrically joins the three circuits corresponding to the clauses variables.
  Since these circuits are placed in $\R^2$, some of them may need to
  cross in order to form some clause.
  This is handled by a \emph{junction} ensuring that key properties
  are maintained when two circuits cross.
  See \Cref{fig:hard_sketch} for an overall sketch of the construction.

  \begin{figure}[h]
    \centering
    \includegraphics[width=6cm]{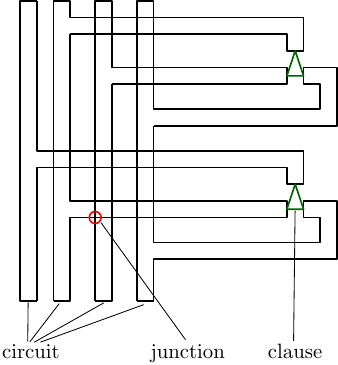}
    \caption{General sketch of the reduction in
    \Cref{thm:kcenter_hard}. Clauses, circuits, and junctions are highlighted.}
    \label{fig:hard_sketch}
  \end{figure}

  We will construct $A$ such that there exists a subset $S \subseteq A$ iwth $I_C(S, A) \leq 4$ if and only if there exists
  a satisfying assignment to the SAT formula. In the following, we will therefore say that $s \in A$ \emph{covers} $a \in A$, if $||s - a|| \leq 4$.
  This means that $I_C(S, A) \leq 4$ if and only if every point in $A$ is covered by some point in $S$.
  We start by describing the circuits. We build a circuit $C_i$ sequentially by placing a sequence of
  points $x,y,z$ and $c$ on a line with $d(x,y) = d(y,z) = 2$, so $d(x,z) =
  4$ and then set the last point $c$ with $d(z, c) = 4$.
  The next triple $x',y',z',c'$ will  be placed such that $d(c, x') = 4$.
  Further points will be placed such that this pattern of four points $x,y,z,c$ repeats until the circuit closes.
  Suppose $C_i$ contains $l_i$ copies of $x, y, z, c$, then we will cyclically label the $t$-th copy of $x,y,z,c$ along $C_i$ by
  $x_i^t, y_i^t, z_i^t, c_i^t$ by fixing $x_i^1, y_i^1, z_i^1, c_i^1$ and an orientation of $C_i$ arbitrarily.
  Often, we will be forced to take corners along a circuit, but this is not an issue for the desired distances as any right-angled turn will
  maintain the desired distances between points.
  For examples of circuits, see \Cref{fig:variable}.

  To explain how an assignment of variable $i$ is encoded by $S \cap C_i$, we claim
  that any subset $S \subseteq C_i$ with $I_{C_i}(S, C_i) \leq 4$ and
  $|S| = l_i$ must consist of either all $x_i^t$ or all $z_i^t$.
  First, observe that no point can cover more than $4$ points of $C_i$.
  In fact, since $|S| = l_i$ and $|C_i|  = 4\cdot l_i$, every point in
  $S$ must cover exactly $4$ unique points of $C_i$ and every point
  in $C_i$ must be covered by exactly one point in $S$.
  Any copy of $y$ or $c$ only covers $3$ points, so they cannot be
  contained in $S$.
  Finally, if $S$ contained some $x$ and some $z$, there would be an
  index $t$ such that $z_i^{t}$ and $x_i^{t+1}$ are both in $S$ and
  then $c_i^{t}$ would be covered twice. So it must hold that $S$
  consists of all $x_i^t$ or all $z_i^t$.
  Having established this, the choice of literal for variable $i$ is encoded by whether $S$ consists of all copies of $x$ or $z$.
  We identify $x$ with a choice of $1$ and $z$ with a choice of $0$ for variable $i$.

  \begin{figure}[h]
    \begin{minipage}[c]{0.4\textwidth}
      \centering
      \includegraphics[width=7cm]{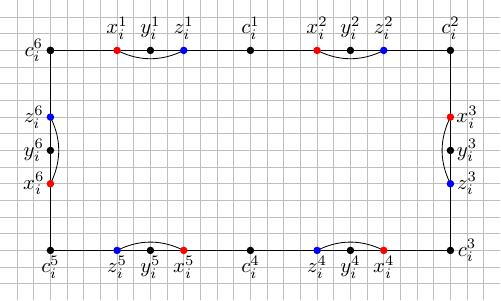}
    \end{minipage}
    \hspace{1cm}
    \begin{minipage}[c]{0.4\textwidth}
      \centering
      \includegraphics[width=7cm]{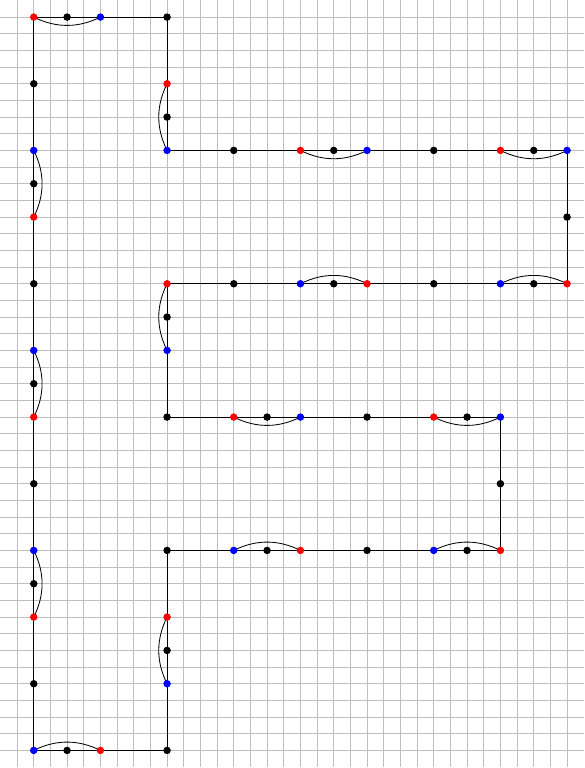}
    \end{minipage}
    \caption{Left: A circuit $C_i$ representing a variable. We draw
      an edge between two points when they can cover each other. This
      circuit has $24 = 4 \cdot 6$ points. The only way to cover all
      points with $6$ points is to select either all $x$ (red), or all
      $z$ (blue). The background grid in $\Z^2$ is displayed. Right: A larger circuit with two arms
    extending to the right.}
    \label{fig:variable}
  \end{figure}

  To represent a clause $W = (l_1 \lor l_2 \lor l_3)$ we join the
  three circuits $C_{i_1}, C_{i_2}, C_{i_3}$ where $C_{i_j}$ is the
  circuit corresponding to  literal $l_j$'s underlying variable $i_j$.
  We introduce a new point $p_W$ with $||p_W - p_{i_j}|| \leq 4$ for
  exactly one $p_{i_j} \in C_{i_j}$ in each of the three circuits.
  Specifically, we construct the gadget in such a way that $p_{i_j} = x_{i_j}^t$
  for some $t$, if setting $i_j$ to $1$ fulfills $W$ and $p_{i_j} =
  z_{i_j}^t$ if setting $i_j$ to $0$ fulfills $W$. See
  \Cref{fig:clause} for how to connect three circuits, such that this
  condition is fulfilled.

  To see that this represents a clause, suppose that we have chosen
  some set $S$ which contains exactly all occurrences of $x$ or all
  occurences of $z$ for each circuit $C_{i_j}$ and that $p_W \notin S$.
  $p_W$ is then covered if and only if there is some circuit
  $C_{i_j}$ for which the unique $v \in C_{i_j}$ with $||p_W - v||
  \leq 4$ is a member of $S$. By construction $v$ must be of type
  $x$ or type $z$ and since we must select either all $x$ or all $z$
  for any circuit, we have made the choice for circuit $C_{i_j}$
  which corresponds to fulfilling literal $l_j$.
  Therefore, this accurately represents the constraint that one of
  the literals in $W$ must be fulfilled.

  \begin{figure}[h]
    \centering
    \includegraphics[width=7cm]{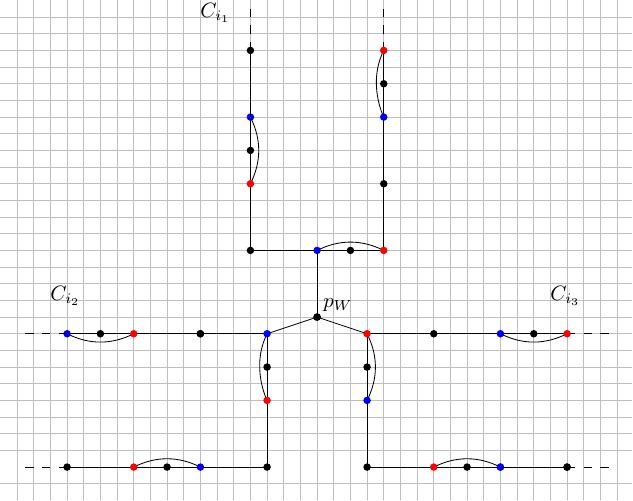}
    \caption{A clause gadget, representing clause $W = (i_1 \lor i_2
      \lor \lnot i_3)$. The three
      circuits $C_{i_1}, C_{i_2}, C_{i_3}$ are combined to form the
    clause gadget.}
    \label{fig:clause}
  \end{figure}

  The two gadgets we have described above almost suffice to complete the proof.
  But as one can see in \Cref{fig:hard_sketch} different circuits may
  need to cross.
  Thus, we place another gadget, a \emph{junction}, at the crossing
  of two circuits which ensures that the property, that one must
  select all occurrences of $x$ or $z$ in each circuit is always maintained.
  See \Cref{fig:junction} for a sketch of such a junction gadget.

  We describe a junction at the crossing of two circuits $C_i$ and$ C_j$,
  where $C_i$ is the circuit going through the junction horizontally
  and $C_j$ goes through the junction vertically. We construct the
  junction directly after some $c_i^s$ and $c_j^t$. Without loss of
  generality we assume that $c_i^s$ is to the left of the junction
  and $c_j^t$ is above the junction. We place four central
  points $p_{x,x}, p_{x,z}, p_{z,x}, p_{z,z}$ on the corners of a square with side length $2$.
  On the right side of the square we proceed with circuit $C_i$,
  placing the next point $c_i^{s+1}$ and continuing the circuit as usual
  by placing the next triple $x_i^{s+2}, y_i^{s+2}, z_i^{s+2}$.
  We proceed analogously for $C_j$ on the bottom with $c_j^{t+1}$ and $x_j^{t+2}, y_j^{y+2}, z_j^{t+2}$.
  Note that the square will be placed such that $||p_{x, x} - c_i^s|| = ||p_{x, z} - c_i^s|| = 4$, similarly
  $||p_{x, x} - c_j^t|| = ||p_{z, x} - c_j^t|| = 4$, $||p_{z, x} - c_i^{s+1}|| = ||p_{z, z} - c_i^{s+1}|| = 4$, and
$||p_{x, z} - c_j^{t+1}|| = ||p_{z, z} - c_j^{t+1}|| = 4$. Every other distance from any $p$ to any other point in a circuit
  will be larger than $4$.

  As the notation suggests, we now want to establish that for a set $S$ of appropriate size covering all points, $p_{x, x} \in S$ exactly if
  in $C_i$ and $C_j$ all $x$ are selected, $p_{z, x}$ is selected exactly if all $z$ are selected in $C_i$ and all $z$ are selected in $C_j$, while $p_{x, z} \in S$ when all $x$ are selected in $C_i$
  while all $z$ are selected in $C_j$, and $p_{z, z} \in S$ exactly if all $z$ are selected in $C_i$ and $C_j$. We only describe the argument for $p_{x, x}$ but the reasoning for the other cases is analogous.

  Suppose $A$ consists of the union of two circuits $C_i$ and $C_j$ and a junction $J$ between the two circuits. $C_i$ (resp. $C_j$) contains $l_i$ (resp. $l_j$) copies of all four $x,y,z,c$ and one more copy of $c$ (the point $c_i^{s+1}$ resp. $c_j^{t+1}$) added by the junction. Together with the $4$ points $p$ in the central square of the junction. So, $ |A| = 4 l_i + 4 l_j + 6$. Now consider some set $S$ with $|S| = l_i + l_j + 1$ that covers all points of $A$. Any point in $C_i$ or $C_j$ covers at most $4$ unique points of $A$.
  Therefore, one of $p_{x,x}, p_{x,z}, p_{z,x}, p_{z,z}$ must be contained in $S$, otherwise not all points would be covered. We assume $p_{x, x} \in S$, the other cases are similar. $p_{x, x}$ covers $p_{x,x}, p_{x,z}, p_{z,x}$ and $p_{z,z}$, as well as $c_i^s$ and $c_j^t$. Since $S$ must contain $l_i + l_j$ additional points and there are $4 \cdot l_i + 4 \cdot l_j$ points left to be covered, any other point must cover exactly $4$ unique remaining points. In particular, $z_i^s$ must be covered.
  One can check that the only point covering $z_i^s$ and 4 unique points in total is $x_i^s$, since $c_i^s$ is already covered. Therefore $x_i^s \in S$.
  Continuing this argument, if $x_i^s \in S$, the only option to cover $z_{i-1}^s$ is $x_i^{s-1}$.
  Repeating this argument then shows that then implies that in $C_i$ all $x$ must be selected. By similar reasoning, also in $C_j$ all $x$ must be selected.

  \begin{figure}[h]
    \centering
    \includegraphics[width=7cm]{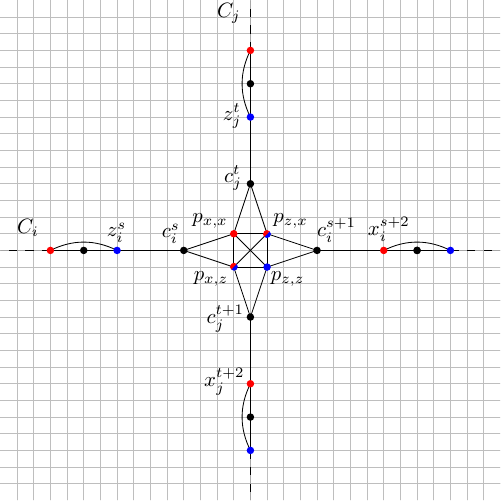}
    \caption{A junction between two circuits $C_i$ and $C_j$. Points
      in the central square are highlighted by which choice of points
      they induce for the circuits. For example: selecting $p_{x, z}$ implies that all $x$ (red) in $C_i$ and all $z$ (blue) in $C_j$ must be selected.
  }
    \label{fig:junction}
  \end{figure}

  To construct the set $A$ from these gadgets we proceed as follows.
  First, insert a series of parallel circuits $C_i$ for each variable $i$.
  They should be tall enough vertically to accomodate a clause gadget for all clauses $W$, so the length should be at least
  $L \cdot |W|$, where $L$ is some large enough constant.
  To make sure that the pattern of $x,y,z,c,$ can loop,  we choose height and width of a circuit such that the overall length of the circuit is divisible by $12$, the total length of one repeating segment of a circuit.
  After this introduce the points $p_W$ in a vertical stack to the right of all the circuits, making sure that the distance between any two $p_W$ is large enough to construct non interfering clause gadgets around each of them.
  Finally, construct each clause gadget by extending a horizontal arm from the three $C_i$ involved in the clause. The arm extending from $C_i$ will cross all circuits to the right of $C_i$,
  before forming the clause gadget around $p_W$.
  We then need to ensure that we place a junction whenever an arm crosses another circuit, and make sure that the closest point to $p_W$ in the arm must be $x$ or $z$ depending on whether setting $i$ to
  $1$ or $0$ fulfills $W$. To ensure these properties, we insert an appropriate number of \emph{creases} (see \Cref{fig:crease}) into straight line segments of a circuit.
  These are minor modifications to the circuit, that don't change the cover relation but slightly shift the cyclic sequence along a sequence. By inserting an appropriate number of creases in a long enough segment, we can ensure that e.g. some $x_i^t$ is always at the desired position, thus allowing us to construct junctions and clauses exactly where needed, without having to consider parity issues in the cyclic sequence.
  Note that for any given gadgets this always requires at most $12$ creases, since the repeating segment $x,y,z,c$ has a length of $12$, so we ensure initially that we always have enough space to place up to $12$ creases between any two junction or clause gadgets.

  \begin{figure}[h]
    \centering
    \includegraphics[width=9cm]{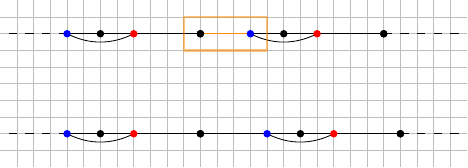}
    \caption{Two circuits running in parallel. In the top circuit a
      crease is highlighted. Note that the cyclic sequence $x,y,z,c$ is
      shifted  forwards by exactly one position compared to the circuit
    without a crease.}
    \label{fig:crease}
  \end{figure}

  Now let $Q = (V, \mathcal{W})$ be a 3SAT formula, and let $A$ be the set of
  points obtained via the described construction.
  Let $C_1, \dots, C_{|V|}$ be the set of circuits in $A$, each
  circuit $C_i$ containing $l_i$ copies of the set of points $x,y,z, c$.
  Let $n_J$ be the number of junctions in $A$.
  We claim that $Q$ is satisfiable if and only if there exists a
  subset of $S$ of $A$ with $|S| = \sum_{i \in V}l_i + n_J$ such that
  $I_C(S, A) \leq 4$.

  For the forward direction let $M$ be a satisfying assignment. For
  each variable $i \in V$, $S$ will contain all $x_i^t$, if $i$ is
  set to $1$ in a satisfying assignment or all $z_i^t$ otherwise.
  In addition, for a junction $J$ between $C_i$ and $C_j$, add
  $p_{v_i, v_j}$, where $v_i = x$ if all $x_i^t$ are selected
  and $v_i = z$, if all $z_i^t$ are selected.
  By construction, all circuits and junctions are then covered. Every
  point $p_W$ is covered, since every clause is fulfilled and we
  therefore have selected one of the points covering $p_W$ for at
  least one of the adjacent circuits.

  For the backward direction, let $S$ cover every point $A$. In
  particular, $S$ must  cover all circuits and junctions. Only at
  most $n_j$ points can cover $6$ unique points, and any point other point
  can cover at most $4$ points of a circuit. Therefore, to cover the
  $4 \sum_{i \in V} l_i + 6 n_J$ points from junctions and circuits in
  $A$, we must select exactly one point in every junction, and $l_i$
  points of circuit $C_i$. The arguments above then imply that we must have
  selected all $x$ or all $z$ for each circuit $C_i$. Indeed, $S$ cannot
   contain any $p_W$, since $p_W$ only covers three
  points lying in any circuit. In turn, this yields a valid
  assignment of the variables of $Q$ based on the selection of
  $x$ or $z$. To finish the proof, observe that, since the $p_W$ must
  be covered, this assignment must be a satisfying assignment. We
  have chosen one of the points covering $p_W$ for every $W$ and
  therefore made an assignment to a literal satisfying clause $W$.
\end{proof}

In the proof above note that the constructed set $A$ only uses
integral points. For uniformity it was not explicitly mentioned in
the original work of \citet{wangkuo1988} that such a construction
only requires integral points, but using such a construction will
turn out to be convenient in future proofs. Therefore we state it formally here.

\begin{corollary}\label{cor:integral}
  \textsc{Discrete $k$-Center} and \textsc{$p$-Dispersion} are
  NP-Complete in $\R^2$ equipped with the Manhattan distance, even if
  all points are integral, placed on a grid whose size is bounded by
  a polynomial in the number of points and a fixed distance threshold of $4$.
\end{corollary}
\begin{proof}
  For \textsc{discrete $k$-Center} it is enough to see that in the proof of
  \Cref{thm:kcenter_hard} we only place integral points and a fixed threshold of $4$.

  For \textsc{$p$-Dispersion} we modify the original proof by
  \citet{wangkuo1988} for the euclidean distance to a proof for the
  manhattan distance, while placing points only on integral coordinates.
  Through analyzing the original proof, it becomes clear that all
  that is required are minor modifications to the circuits,
  junctions, and clauses. We show how to modify these constructions
  in \Cref{fig:integral_sketch}.

  \begin{figure}
    \centering
    \begin{minipage}[c]{.25\textwidth}
      \includegraphics[width=0.9\textwidth]{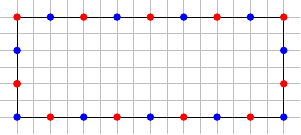}
    \end{minipage}
    \hfill
    \begin{minipage}[c]{.45\textwidth}
      \includegraphics[width=0.9\textwidth]{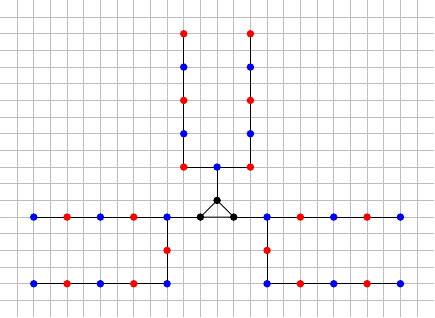}
    \end{minipage}
    \hfill
    \begin{minipage}[c]{.25\textwidth}
      \includegraphics[width=0.95\textwidth]{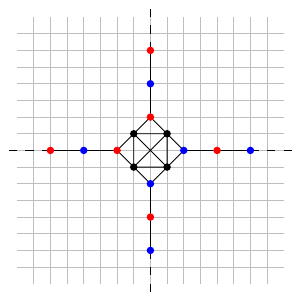}
    \end{minipage}
    \caption{Circuit, Clauses, and junctions in the modified proof
      that $p$-dispersion is NP-Complete in Dimension $2$ under the
    manhattan distance. The goal is to find a subset $S$ with $I_U(S) \geq 4$.}
    \label{fig:integral_sketch}
  \end{figure}

  To see that the highest occurring coordinate (assuming all
  coordinates are positive) of some point is bounded by a polynomial,
  note that the coordinates horizontally are bounded by a linear
  function in the number of circuits and the coordinates vertically
  are bounded by a linear function in the number of clauses, both of
  which are smaller than the total number of points.
\end{proof}

\subsection{Hardness for Pareto Pruning in fixed dimension}

\dimthreeprune*
\begin{proof}
  We denote the scalar product between vectors $x,y \in \R^d$ by
  $\langle x, y \rangle = \sum_{i = 1}^d x_i \cdot y_i$. Let $\veps >
  0, n = (1, \veps, \veps) \in \R^3$. Let $H \subseteq \R^3$ be the
  hyperplane orthogonal to $n$ and let $x, y \in H$.
  We claim that neither $x$ dominates $y$ or $y$ dominates $x$:
  Without loss of generality suppose $x$ dominates $y$, then
  \[ 0 < (x_1 - y_1) + \veps (x_2 - y_2) + \veps (x_3 - y_3) =
  \langle x, n \rangle - \langle y, n \rangle.\]
  This implies that one of $\langle x, n \rangle$ or $\langle y, n
  \rangle$ must not be $0$, proving the claim by contraposition.

  Now note that $e_1 \coloneq (-\veps, 0, 1)$ and $e_2 \coloneq
  (-\veps, 1, 0)$ form a basis for $H$. For some $A \in \R^2$ consider the map
  $f: A \to H, (x_1, x_2) \mapsto x_1 e_1 + x_2 e_2$. To show
  hardness for \textsc{Uniformity Pareto Pruning} and
  \textsc{Coverage Pareto Pruning}, we want to adapt the hardness
  proofs from \Cref{cor:integral}.

  For coverage, let $Q = (V, \mathcal{W})$ be a SAT instance.
  Construct the set of points $A_Q$ in $\R^2$, as in the proof of
  \Cref{thm:kcenter_hard}. Now, setting $\veps = \frac{1}{8}$ and
  mapping $A_Q$ through $f$ we claim that $I_C(S, A_Q) \leq 4 \iff
  I_C(f(S), f(A_Q)) \leq \frac{9}{2}$.

  Note that all points in the reduction in \Cref{thm:kcenter_hard}
  are integral, thus, their distances are also integral. Therefore,
  if $I_C(S, A_Q) > 4$, it follows that $I_C(S, A_Q) \geq 5$. Now for
  any $x, y \in A_Q$ with $||x - y|| \leq 4$ we get
  \begin{multline*}
    ||f(x) - f(y)|| = \veps | x_1 - y_1 + x_2 - y_2 | + |x_1 - y_1| +
    |x_2 - y_2| \leq \veps ||x - y|| + ||x - y|| = \left(1 +
    \frac{1}{8}\right) 4 = \frac{9}{2}
  \end{multline*}
  on the other hand for $x, y \in A_Q$ with $||x - y|| \geq 5$, we get
  \[||f(x) - f(y)|| \geq ||x - y|| \geq 5\]
  therefore $I_C(S, A_Q) \leq 4$ if and only if $I_C(f(S), f(A_Q))
  \leq \frac{9}{2}$.
  Since no two points in $f(A_Q)$ dominate another, this means that
  $f(A_Q)$ is a valid input for \textsc{Coverage Pareto Pruning} and
  therefore $Q$ is satisfiable if and only if there exists some
  subset $S' \subseteq f(A_Q)$ such that $I_C(S', A_Q) \leq \frac{9}{2}$.

  To show hardness for \UPP{} we employ a similar argument using the
  reduction in \Cref{cor:integral}. Here, we have that a SAT instance
  $Q = (V, \mathcal{W})$ is satisfiable if and only if $I_U(S, A_Q)
  \geq 4$. Similarly to above we get that $||x - y|| \geq 4$ implies
  $||f(x) - f(y)|| \geq 4$ and $||x -y || \leq 3$ implies $||f(x) -
  f(y)|| \leq \veps || x - y || + || x - y|| \leq 3 + \frac{3}{8} < 4$.
  Therefore $Q$  is satisfiable if and only if there exists a subset
  $S \subseteq A_Q$, such that $I_U(S) \geq 4$.
\end{proof}

\dcdimtwoeasy*
\begin{proof}
  We present an algorithm for two objectives. For one objective, the
  problems are trivial, since there is only one unique Pareto optimal
  solution. Let $A = \{a_1, a_2, \dots, a_n\}$ be the set of Pareto
  optimal alternatives, where $a_i = (x_i, y_i)$ and $x_1 \leq x_2
  \leq \dots \leq x_n$.
  Since all points in $A$ are Pareto optimal this implies $y_1 \geq
  y_2 \geq \dots \geq y_n$. Also define $A_i \coloneq \{a_1, \dots, a_i\}$
  for $i \in [n]$.
  Let $\delta_{ij} \coloneq \max_{j < t < i} \min(||a_t - a_i||_+,
  ||a_t - a_j||_+)$.

  For $i \in [n]$, declare a table $T$ via
  \[T[i, l] \coloneq \min_{S \subseteq A_i, |S| = l, a_i \in S}
  I_{DC}(S, A_i). \]
  A recursion for $T$ is given by
  \[T[i, l] = \min_{j = 1}^{i-1} \max(T[j, l-1], \delta_{ij}).\]
  To see why, let $S_i \in \argmin_{S \subseteq A_i, |S| = l, a_i \in
  S} I_{DC}(S, A_i)$, $j = \max\{j \mid a_j \in S_i \setminus
  \{a_i\}\}$ and $S_j = S_i \setminus \{a_i\}$.
  Observe that $I_{DC}(S_i, A_i) = \max\{I_{DC}(S_j, A_j),
  \delta_{ij}\}$. The idea is that for elements $a_{j'}$ with $j' <
  j$ the closest element of $S_i$ cannot be $a_i$, since $a_j \in
  S_i$ and $j' < j$. The $\delta_{ij}$ term then accounts for the cost of the remaining elements that still
  need to be covered.
  To initialize $T$, we set $T[i, 1] = ||a_1 - a_i||_+$ for all $i \in [n]$.

  Note that we can recover the optimal objective value from $T$ using
  the following formula: $\min_{S \subseteq C, |S| = k} I_{DC}(S, A)
  = \min_{i \in [n]} \max\{T[i,k], x_n - x_i\}$:
  We iterate over all possible options for a rightmost point of $S$
  and then compare the optimal solutions.

  Filling the table naively using dynamic programming we first need to sort the $a_i$ by
  their $x$ coordinate. Then we require a runtime of $n^3$ to determine
  all $\delta_{ij}$. $T$ has $n\cdot k$ entries, each of which takes
  time $O(n)$ to fill. Thus we get an overall runtime of $O(n^3 + n^2k)$.
  This runtime can be improved by employing the techniques
  \citet{DBLP:journals/cor/VazPFKS15} use to refine the runtime in
  their algorithm for \textsc{Coverage Pareto Pruning}, to yield an improved total runtime of
  $O(nk + n\log n)$.
\end{proof}

\begin{lemma}\label{lem:directed_trigrid}
    Let $\Gamma = (V, E)$ be the graph of the infinite hexagonal grid (see \Cref{fig:tri_sketch}) with 
    {\footnotesize \begin{itemize}
        \item $V = \{(i,j) \mid i, j \in \Z\}$,
        \item $E = \left\{\bigl((i,j), (i',j')\bigr) \mid (i,j) \in V, (i',j') \in \{(i+1,j), (i+1, j-1), (i, j-1), (i-1, j), (i-1,j+1), (i, j+1)\}\right\}$.
    \end{itemize}}
    Denote the length of a shortest path between two vertices $v,w$ in $\Gamma$ by $d_\Gamma(v, w)$.
    There is an embedding $f:V \to \R^3$ such that $d_\Gamma(v, w) = ||f(v) - f(w)||_+$.
\end{lemma}
\begin{proof}
    \begin{figure}
        \centering
        \includegraphics[width=6cm]{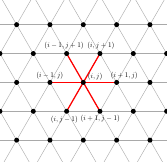}
        \caption{An excerpt of the triangular grid $\Gamma$. The neighborhood of $(i, j)$ is highlighted.}
        \label{fig:tri_sketch}
    \end{figure}
    \citet{luczak1976} have characterized $d_\Gamma$ as
    \[d_\Gamma((i,j), (h,k)) = \begin{cases} |i-h| + |j - k|, &\text{if } i-h \text{ and } j-k \text{ have the same sign}, \\ \max\{|i-h|, |j-k|\}, &\text{ otherwise}.\end{cases}\] 
    Let $e_1 = (1, 0, -1), e_2 = (0, 1, -1) \in \R^3$. We claim that the map $f:V \to \R^3, (i, j) \mapsto i \cdot e_1 + j \cdot e_2$ fulfills the desired properties. Let $(i, j), (h, k) \in V$ such that $i-h$ and $j -k$ have the same sign, then 
    \begin{align*}
        & ||f\bigl((i, j)\bigr) - f\bigl((h, k)\bigr)||_+ = ||i \cdot e_1 + j \cdot e_2 - h \cdot e_1 - k \cdot e_2||_+ \\
        &\qquad = \max\{0, i - h\} + \max\{0, j-k\} + \max\{0, h - i + k - j\}
        = |i - h| + |j-k|,
    \end{align*}
    where the last equality is true as either exactly the first two remain or the last term remains if $i -h$ and $j-k$ have the same sign. If $i - h$ and $j-k$ have different signs then
    \begin{equation*}
        \max\{0, i - h\} + \max\{0, j-k\} + \max\{0, h - i + k - j\} = \max\{|i -h|,|j-k|\}.
    \end{equation*}
    To see that the equality holds, without loss of generality let $| i - h| > |j -k|$. If $i - h > 0$, then all but the first term disappears. If $i - h < 0$, then the two last terms remain and sum up to $h - i = |i - h|$.
\end{proof}
\begin{lemma}\label{lem:trigrid_hard}
    \textsc{Discrete $k$-Center} is NP-Complete, if restricted to the Metric $d_\Gamma$ and $A \subseteq \Z^2$.
\end{lemma}
\begin{proof}
    We employ a reduction from \textsc{3SAT} similar to the proof of \Cref{thm:kcenter_hard}. The key difference here is that, instead of considering points in $\R^2$ with the manhattan distance,
    distances are now given by $d_\Gamma$.
    To prove the theorem we therefore modify the gadgets and distance threshold from \Cref{thm:kcenter_hard} to comply with this modified distance metric.
    For $s, a \in \Z^2$ we will say $s$ covers $a$ whenever $d_\Gamma(s, a) \leq 2$. In general, we construct circuits, clauses, junctions and creases similarly to \Cref{thm:kcenter_hard},
    but they need to be modified to work with $d_\Gamma$ and the new threshold. In \Cref{fig:tri_gadget} we display how to construct each of these gadgets under $d_\Gamma$.
    Take note that, for each gadget, the graph of points which cover each other is identical to the gadgets employed in \Cref{fig:tri_gadget}.

    Now, given a \textsc{3SAT} formula, we can follow the general construction shown in \Cref{fig:tri_gadget} to construct a set of points $A$ by using the modified gadgets.
    Since the gadgets are analogous to those in \Cref{fig:tri_gadget}, we can reason about the gadgets in the same fashion to establish that there exists an integer $k$ such that the SAT formula has a solution, if and only if there is a subset $S \subseteq A$ such that $\max_{a \in A} \min_{s \in S} d_{\gamma}(s, a) \leq 2$. 
\begin{figure}[h]
    \centering
    \includegraphics[width=5cm]{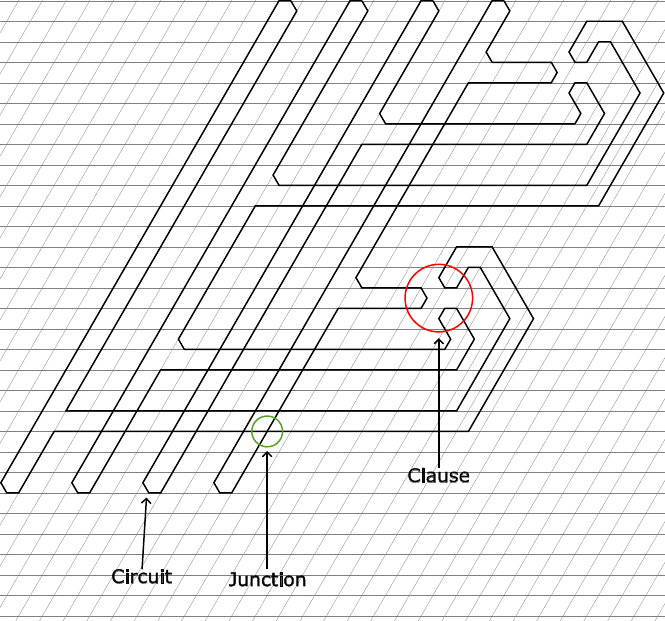}

    \begin{minipage}[c]{.2\textwidth}
        \includegraphics[width=3cm]{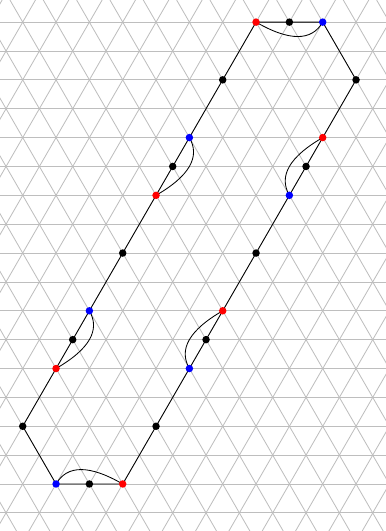}
    \end{minipage}
    \hfill
    \begin{minipage}[c]{.2\textwidth}
        \includegraphics[width=3cm]{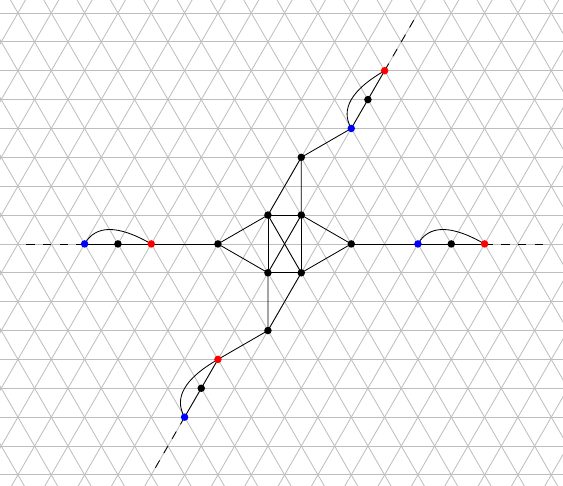}
    \end{minipage}
    \hfill
    \begin{minipage}[c]{.2\textwidth}
        \includegraphics[width=3cm]{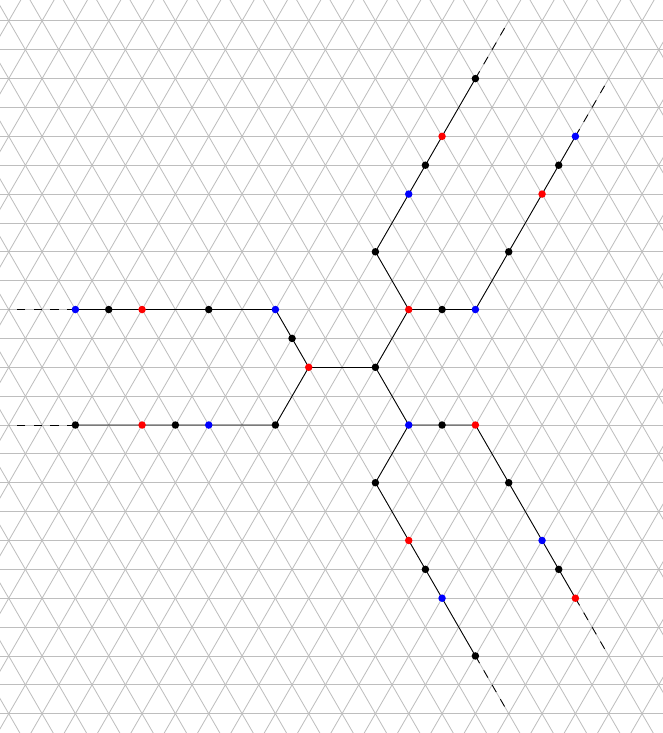}
    \end{minipage}
    \begin{minipage}[c]{.2\textwidth}
        \includegraphics[width=3cm]{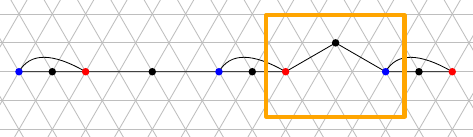}
    \end{minipage}
    \caption{Top: A sketch of the overall reduction in \Cref{lem:trigrid_hard}. Bottom: Circuits, Junctions, Clauses, and Creases in the reduction.}
    \label{fig:tri_gadget}
\end{figure}
\end{proof}
\dcdimthreehard*
\begin{proof}
    This is a direct consequence of \Cref{lem:directed_trigrid} and \Cref{lem:trigrid_hard}. Given a 3-SAT Formula $Q$, construct a set of points $A$ and integer $k$ as outlined in the proof of \Cref{lem:trigrid_hard}. Then, by applying the map $f$ from \Cref{lem:directed_trigrid}, we get that $Q$ has a solution if and only if there exists a subset $S \subseteq f(A)$, with $|S| = k$ such that $I_{DC}(S, f(A)) \leq 2$. Thus, directed coverage is NP-Complete for three objectives.
\end{proof}

\subsection{Ordinal Objectives}
\ucordinalhard*
\begin{proof}
  We show the statement for Coverage and Uniformity in
  \Cref{uc_ord_hard} and for Directed Coverage in \Cref{dc_ord_hard}.
\end{proof}

\begin{lemma}\label{lem:hard_map}
  Let $\Delta \in \N$. For any $n \in \mathbb{N}$, there is an
  injective map $f:[n]\times [n] \to \Z^2$ and a $\Delta'$, such that
  $||x - y|| \leq \Delta \iff ||f(x) - f(y)|| \leq \Delta'$ and no
  two points share a coordinate under $f$.
\end{lemma}
\begin{proof}
  Let $t = \max\{n+1, 2 \Delta+2\}$ let $f:[n] \times [n] \to \Z^2,
  (x_1,x_2) \mapsto (tx_1 + x_2, x_1 + tx_2)$ and $\Delta' = (t+1)\Delta$.
  No two points share a coordinate under $f$: suppose $f_1(x) = f_1(y)
  \iff t x_1 + x_2 = ty_1 + y_2$.
  Because $t > n$, taking the modulo of $t$ implies $x_2 = y_2$,
  leaving with us $t x_1 = t y_1$, so $x_1 = y_1$. An analogous argument can be used to show that $f_2(x) = f_2(y)$.

  Furthermore for $x,y \in \Z^2$ $||x - y|| \leq \Delta$ implies $|x_1 - y_1| + |x_2 -
  y_2| \leq \Delta$ and therefore
  \begin{align*}
    ||f(x) - f(y)|| & = |t(x_1 - y_1) + x_2 - y_2| + |t(x_2 - y_2) +
    x_1 - y_1|  \\
    & \leq t|x_1 - y_1| + |x_2 - y_2| +  t|x_2 - y_2| + |x_1 - y_1| \\
    &=  (t+1)(|x_1 - y_1| + |x_2 - y_2|) \\
    & \leq (t+1)\Delta = \Delta'.
  \end{align*}
  On the other hand, if $||f(x) - f(y)|| \leq \Delta' = (t+1)\Delta$, then
  \begin{align*}
    (t+1)\Delta &\geq ||f(x) - f(y)|| \\
    &=  |t(x_1 - y_1) + x_2 - y_2| + |t(x_2 - y_2) + x_1 - y_1| \\
    &\geq t|x_1 - y_1| - |x_2 - y_2| + t|x_2 - y_2| - |x_1 - y_1|\\
    & = (t - 1)||x-y||.
  \end{align*}
  This then implies $ ||x - y|| \leq \frac{t+1}{t-1}\Delta$. Using $t
  \geq 2\Delta + 2$, then gives
  $||x - y||\leq \frac{2 \Delta + 3}{2 \Delta + 1} \cdot \Delta =
  \Delta + \frac{2 \Delta}{2 \Delta + 1} < \Delta + 1$. Since $x$ and $y$ are integral , so must be $||x-y||$ and therefore $||x -y || \leq \Delta$. 
  So $f$ fulfills all desired properties.

  Also see \Cref{fig:coord_transform} for a sketch of $f$.

  \begin{figure}
    \centering
    \scalebox{.5}{\begin{tikzpicture}[
    scale=0.5,
    dot/.style={circle, fill=black, draw, minimum size=3.5pt, inner sep=0pt}
    ]

\def\gridspacing{1} 

\begin{scope}
\foreach \x in {0,1,...,8} {
    \foreach \y in {0,1,...,8} {
        \fill[blue!50] (\x*\gridspacing, \y*\gridspacing) circle (3pt);
    }
}

\coordinate[] (a) at (8.5, 4);
\coordinate[] (b) at (10.5, 4);
\end{scope}

\begin{scope}[xshift=11cm, yshift=-8cm, scale=.33]
\foreach \x in {0,1,...,8} {
    \foreach \y in {0,1,...,8} {
        \fill[blue!50] (8*\x*\gridspacing+\y, \x + 8*\y*\gridspacing) circle (9pt);
    }
}
\foreach \x in {0,1,...,72} {
    \foreach \y in {0,1,...,72} {
        \fill[black!50] (\x*\gridspacing, \y*\gridspacing) circle (3pt);
    }
}

\draw[->] (a) -- (b);
\end{scope}
\end{tikzpicture}}
    \caption{The transformation $f$ in the proof of
      \Cref{lem:hard_map}. Note that after applying the transformation,
    no two points share any coordinate.}
    \label{fig:coord_transform}
  \end{figure}
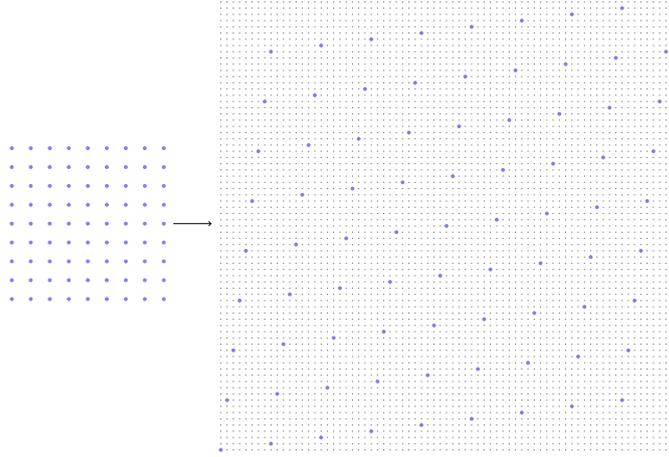

\end{proof}

\begin{proposition}\label{uc_ord_hard}
  \textsc{Uniformity / Coverage Pareto Pruning} is NP-hard, even if
  every objective is ordinal.
\end{proposition}

\begin{proof}
  We prove this statement via a reduction from $3$-SAT. We will only
  formally state the proof for coverage, but the proof for uniformity
  will rely on the same ideas.
  Let $Q$ be a SAT-Formula and $A_Q$ be the set of points constructed
  in the proof of \Cref{thm:kcenter_hard}.

  We construct a set of alternatives $X$ and a set of ordinal
  objective functions on $X$.
  We set $X \coloneq A_Q \cup D \cup \{x_*\}$, where $D$ is a set of
  dummy alternatives and $x_*$ is an additional distinguished
  alternative. Using the observations from \Cref{cor:integral} we
  assume that $A_Q \subseteq [M] \times [M]$ for some $M$ whose size
  is bounded by a polynomial in $|A_Q|$. Let $g, \Delta'$ be the map
  and threshold from \Cref{lem:hard_map} acquired for $n = M$ and $\Delta = 4$.

  We describe four different kind of ordinal objectives. Let $s, m
  \in \mathbb{N}$ be arbitrary for now. We will specify later how to
  choose them. $f_1: X \to [|X|]$ will map elements as follows:
  $f_1(x_*) = M + s$ and for $a \in A_Q$, $f_1(a) = g_1(a)$. Now
  $f_1$ restricted to $A_Q \cup \{x_*\}$ is injective, since no two
  points share a coordinate under $g$. To make it so that $f$ maps
  bijectively to $[|X|]$ we use the dummy alternatives in $D$ to fill the remaining
  positions arbitrarily. $f_2$ will be defined similarly, with
  $f_2(x_*) = M + s$, but for $a \in A_Q$, $f_2(a) = g_2(a)$. Again
  we fill the remaining positions with elements of $D$.

  The last two objectives are $f_\ell$ and $f_r$. $f_\ell$ maps $A_Q$
  to $\{M+s - |A_Q| + 1, \dots, M+s\}$ in some arbitrary way and
  $f_\ell(x_*) =M + s - |A_Q|$. $f_r$ is similar to $f_\ell$, except
  $A_Q$ is mapped to $\{M+1 - |A_Q| + 1, \dots, M+s \}$ in the
  reverse order of $f_1$. So for $a, b \in A_Q$, whenever $f_\ell(a)
  < f_\ell(b)$ we get $f_r(a) > f_r(b)$. Again, we complete $f_\ell$
  and $f_r$ using elements of $D$ such that they are bijective.
  Intuitively, the top positions under $f_\ell$ and $f_r$ are taken
  up by $A_Q$, followed by $x_*$ and then all of the $D$.

  To construct the total set of objectives now take $m \in \N$ copies
  of $f_1$ and $f_2$ each, and a single copy of $f_\ell$ and $f_r$ each.

  Now, every dummy alternative in $D$ is dominated by $x_*$, since
  $x_*$ achieves the highest objective value in $f_1$ and $f_2$ and a
  higher objective than any dummy alternative in $f_\ell$ and $f_r$.
  Hence, all dummy alternatives must be ignored for Pareto pruning.
  More so, no other alternative is Pareto dominated, since $a \in A_Q$ appears before
  $x_*$ in $f_l$ and $f_r$, and any pair $a, b \in A_Q$ is ranked in reverse orders
  by $f_l$ and $f_r$.
  Furthermore, for $a, b \in A_Q$ we get $|f_1(a) - f_1(b)| + |f_2(a)
  - f_2(b)| = |g_1(a) - g_1(b)_1| + |g_2(a)_2 - g_2(b)| = ||g(a) - g(b)||$.
  Since we have $m$ copies of $f_1$ and $f_2$, writing $f$ for the aggregate function of all objectives, we get
  \begin{align*}m \cdot ||g(a) - g(b)|| + 2 &\leq ||f(a) - f(b)|| \\ & =   m \cdot ||g(a)
    - g(b)|| + |f_\ell(a) - f_\ell(b)| + |f_r(a) - f_r(b)| 
    \leq m \cdot
  ||g(a) - g(b)|| + 2|A_Q|\end{align*} for $x_*$ we get $|f_1(a) - f_1(x_*)|
  \geq s$, so $||f(a) - f(x_*)|| \geq m \cdot s$.

  We now set $m \geq 2 |A_Q|$. For $a,b \in A_Q$ we now claim that
  $||a - b|| \leq \Delta$, if and only if $||f(x) - f(y)|| \leq m
  \cdot (\Delta' + 1)$.
  Let $||a - b|| \leq \Delta$, then \[||f(x) - f(y)|| \leq m \cdot
  ||g(a) - g(b)|| + 2 |A_Q| \leq m \cdot \Delta' + m = m \cdot (\Delta' + 1),\]
  On the other hand, if $||f(a) - f(b)|| \leq m \cdot (\Delta' +1)$, then
  \[ m \cdot (\Delta'+1) \geq || f(a) - f(b) || \geq m \cdot||g(a) -
  g(b)|| + 2 \implies \Delta' + 1 - \frac{2}{m} \geq ||g(a) - g(b)||.\]
  Since $||g(a) - g(b)||$ is integral this implies $||g(a) -g(b)||
  \leq \Delta'$ and therefore $||a - b|| \leq \Delta$.
  Additionally, notice that for $x_*$, by choosing $s > \Delta' + 1$ we get $||f(x_*) - f(a)||
  \geq m \cdot s > m \cdot (\Delta'+1)$ for all $a \in A_Q$.

  Let $f(X)_{PO}$ be the set of Pareto optimal points in $f(X) =
  \{f(x) \mid x \in X\}$. We now claim that that there exists a
  subset $S \subseteq f(X)_{PO}$ with $|S| = k+1$ and $I_{C}(S,
  f(X)_{PO}) \leq m (\Delta' + 1)$ if and only if the SAT formula $Q$
  is satisfiable.

  To establish this correspondence, first consider that, since
  $||f(x_*) - f(a)|| > m \cdot (\Delta' + 1)$, any such $S$
  must contain $f(x_*)$, otherwise $f(x_*)$ would not be covered. For the
  remaining $k$ points $S' = S \setminus \{f(x_*)\}$ we get that
  $I_C(S', f(A_Q)_{PO}) \leq m \cdot (\Delta' + 1)$ if and only if
  $I_C(S, A_Q) \leq \Delta$, since $||f(a) - f(b)|| \leq m \cdot
  (\Delta' + 1)$ if and only if $|| a -b || \leq \Delta$. Finally,
  $I_C(S, A_Q) \leq \Delta$ if and only if $Q$ is satisfiable by the
  construction of $A_Q$, completing the proof.

  The proof for uniformity is identical, other than using the
  observation that including $f(x_*)$ in $S$ never causes
  $I_U(S)$ to sink below $m \cdot (\Delta' + 1)$.
\end{proof}

\begin{proposition}\label{dc_ord_hard}
  \DCPP{} is NP-Hard, even if every objective is ordinal.
\end{proposition}
\begin{proof}
  We show hardness via reduction from \textsc{Exact Cover by 3-Sets}(X3C),
  with the additional assumption that every element appears in exactly three sets.
  Let $E$, $\mathcal{S} \subseteq 2^E$ constitute some X3C instance.
  For some $e \in E$ let $\mathcal{S}_e \coloneq \{S \in \mathcal{S}
  \mid e \in S\}$. By the above assumption $|\mathcal{S}_e| = 3$ for all $e \in E$.

  We declare ordinal objectives $f_e$ for every $e \in E$ and one
  additional objective $f_*$. For every element $e \in E$ we
  introduce a set of $6$ dummy alternatives $D_e$. We introduce an
  additional high-quality alternative $x_*$ and a corresponding set
  of $9$ dummy alternatives $D_*$.
  The set of alternatives $X$ will be given by $X \coloneq E \cup
  \mathcal{S} \cup \{x_*\} \cup D$, where $D = D_* \cup \bigcup_{e \in E} D_e$.
  So $|D| = 6|E| + 9$.

  Because the objectives are ordinal, an objective is fully
  determined by its ordering of alternatives.
  We only describe the order in which some objective ranks the
  alternatives, instead of specifying the values $f_e(a)$ explicitly.
  For alternatives $a, b$ we write $a \succ_e b$ (resp. $a \succ_*
  b$), if $f_e(a) > f_e(b)$, (resp. $f_*(a) > f_*(b)$). We also
  extend this notation to sets, so $a \succ B$ for some $B \subseteq
  A$ means that $a$ is ranked higher than any element of $B$.

  For $f_e$ we construct the ordering
  \[ e \succ_e  D_e \succ_e \mathcal{S}_e \succ_e x_* \succ_e
    \mathcal{S} \setminus \mathcal{S}_e \succ_e E \setminus \{e\}
  \succ_e D \setminus D_e\]
  and for $f_*$ we define the ordering
  \[x_* \succ_* D_* \succ_* \mathcal{S} \succ_* E \succ_* D \setminus D_*.\]
  Now define $A \coloneq \{f(x) \mid x \in X, x \text{ is Pareto optimal}\}$.
  Let $\ell = \frac{|E|}{3}$. We claim that there exists a subset $T
  \subseteq A$ with $|T| = \ell + 1$ and $I_{DC}(T, A) \leq 9$ if and
  only if there exists a solution to the X3C instance. First observe
  that all alternatives in $D$ are Pareto dominated. Every
  alternative in $D_e$ is dominated by $e$, every alternative in
  $D_*$ is dominated by $x_*$ and every other alternative in $x$ is
  Pareto optimal, so the set of Pareto optimal alternatives $A = f(E \cup \mathcal{S}
  \cup \{x_*\}).$

  Now let $\{S_1, \dots, S_\ell\} \subseteq \mathcal{S}$ with
  $\bigcup_{i = 1}^\ell S_i = E$ be a solution to X3C. Let $T = \{S_1,
  \dots, S_\ell\} \cup \{x_*\}$, we claim $I_{DC}(f(T), A) \leq 9$.
  We show this by proving that for every $a \in A$ there exists some $t
  \in T$ such that
  $||a - f(t)||_+ \leq 9$.
  For $a = f(S)$, we get $||a - f(x_*)|| \leq 9$, as $S = \{e_1, e_2,
  e_3\}$ is only ranked higher than $x_*$ in $f_{e_1}, f_{e_2},
  f_{e_3}$ and $S$ is ranked at most three positions ahead of $x_*$ in
  each of these objectives. If $a = f(e)$ for some $e \in E$, since $\{S_1, \dots,
  S_\ell\}$ is a X3C solution, there is some $i$ such that $e \in S_i$
  and $S_i \in \mathcal{S}_e$. We then get $||a - f(S_i)|| = f_e(e) -
  f_e(S_i) \leq |D_e| + |\mathcal{S}_e| = 9$, as $e$ appears behind $S_i$ in every objective
  other than $f_e$. This completes the forward direction.

  Now let $T \subseteq A$ with $I_C(f(T), A) \leq 9$ and $|T| = \ell + 1$.
  Suppose $T$ contains $x_*$ and only alternatives from $\mathcal{S}$.
  $I_C(f(T), A) \leq 9$ implies that for all $e \in E$, there must be
  some $S' \in T$ such that $f_e(e) - f_e(S') \leq 9$, because $f_e(e)
  - f_e(x_*) = 10$. By construction the only such $S'$ are in
  $\mathcal{S}_e$ so the sets in $T \cap \mathcal{S}$ form a set cover.

  It remains to justify the assumption that $T$ must contain $x_*$ and
  only alternatives of $\mathcal{S}$. In $f_*$ there are no Pareto optimal
  alternatives $a$ with $f_*(x_*) - f_*(a) \leq 10$, since $D_*$ is
  dominated by $x_*$. As such, $T$ must contain $x_*$.
  Finally, suppose $T$ contains $x_*$ and at least one alternative in
  $E$. It follows that there at most $\ell - 1 = \frac{|E|}{3} - 1$
  alternatives of $\mathcal{S}$ in $T$. The number of $e \in E$ such
  that there exists $a \in T$ with $f(e)_e - f_e(a) \leq 9$ is at most
  $3 (\ell - 1) + 1 = |E| - 2 < |E|$, since every $S \in \mathcal{S}$
  covers at most three $e \in E$. Hence $T$ cannot contain $e \in E$
  and the statement follows.
\end{proof}

\subsection{Approval Objectives}

\ucapprovalhard*

\begin{proof}
  We reduce from \textsc{Dominating Set} for coverage and directed
  coverage. For uniformity we reduce from \textsc{Independent Set}.
  Let $G = (V, E)$ be a graph and $k$ be an integer.

  Construct a set of alternatives $X = V$. For each $e \in E$ construct
  the following objectives: One objective $f_e$, with $f_e(v) = 1$, if
  $v \in E$ and $f_e(w) = 0$, otherwise.
  Add objectives $f_v^e$ for each $v \notin e$ with $f_v^e(v) = 1$ and
  $f_v^e(w)=0$, for all $w \neq v$. We define $|| v- w||_e = |f_e(v) -
  f_e(w)| + \sum_{u \notin e} |f_u^e(v) - f_u^e(w)|$.
  For some fixed $e \in E$ and $v, w \in e$, we get
  \[||v -w||_e = |f_e(v) - f_e(w)| + \sum_{u \notin e} |f_u^e(v) -
  f_u^e(w)|= |1 - 1| + \sum_{u \notin e}|0 - 0| = 0.\]
  If $v \in e$, $w \notin e$, we get
  \[||v -w||_e = |f_e(v) - f_e(w)| + |f^e_w(v) - f_w^e(w)| + \sum_{u
    \notin e, u \neq w} |f_u^e(v) - f_u^e(w)|= |1 - 0| + |0 - 1| \sum_{u
  \notin e, u \neq w}|0 - 0| = 2.\]
  Lastly, for $v, w \notin E$ it follows
  \begin{align*}||v -w ||_e &= |f_e(v) - f_e(w)| + |f^e_v(v) -
    f_v^e(w)| + |f^e_w(v) - f_w^e(w)| + \sum_{u \notin e, u \notin
    \{v,w\}} |f_u^e(v) - f_u^e(w)|\\ & = |0 - 0| + |1 - 0| + |0 - 1| +
    \sum_{u \notin e, u \notin \{v,w\}}|0 - 0| = 2.
  \end{align*}
  In aggregate, this means that $||v - w||_e = 0$, if $\{v, w\} \in E$
  and $2$ otherwise.

  For the overall distance it then follows that
  \[
    ||f(v) - f(w)|| = \sum_{e \in E} ||v - w||_e =
    \begin{cases} 2 |E| - 2, &\text{if } \{v, w\} \in E \\ 2 |E|
      &\text{otherwise}
    \end{cases}
  \]
  Let $A = \{f(X) \mid x \in X\}$. It is easy to see that no
  alternative in $A$ is Pareto dominated. We claim that $G$ has a
  dominating set of size $k$, if and only if $A$ admits a
  slate $S \subseteq k$ for coverage with $|S| = k$ and $I_C(S, A) \leq 2 |E|-2$.
  Let $D \subseteq V$ be a dominating set, then $I_C(S, A) = 2 |E| -
  2$, as for every $v \in V$, there must be some $d \in D$, such that
  $\{v, d\} \in E$ and therefore $||f(v) - f(d)|| \leq 2 |E| - 2$, so
  $I_C(f(D), A) \leq 2 |E| - 2$.
  Conversely, let $S \subseteq A$ be a slate with $I_C(S, A) \leq 2 |E|
  - 2$, then for every $v \in V$ there must exist some $w \in X$ such
  that $f(w) \in S$ and $||f(v) - f(w)|| = 2 |E| -2 $, so $v$ and $w$
  are neighbours. Therefore $f^{-1}(S)$ is a dominating set.

  For directed coverage, note that, for the instances we have
  constructed here $||f(v) - f(w)||_+ = \frac{||f(v) - f(w)||}{2}$. So an analogous
  argument applies.
  The hardness of uniformity is shown by observing that $I_U(S) \geq 2
  |E|$, if and only if $S$ is an independent set.
\end{proof}

\newpage

\section{Additional Material for Experiments}\label{app:experiments}

\begin{figure}[h]
  \centering
  \includegraphics[width=0.8\linewidth]{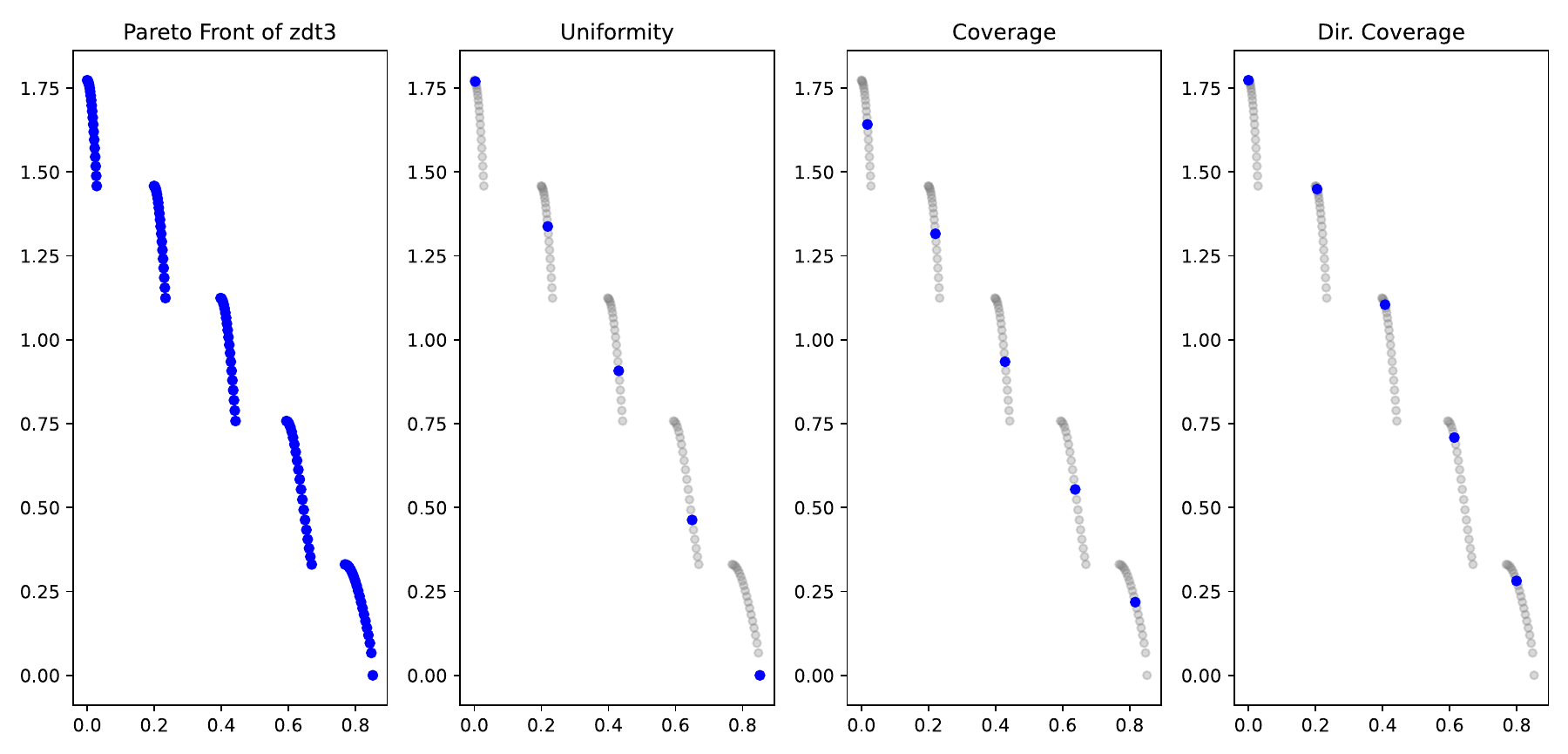}
  \includegraphics[width=0.8\linewidth]{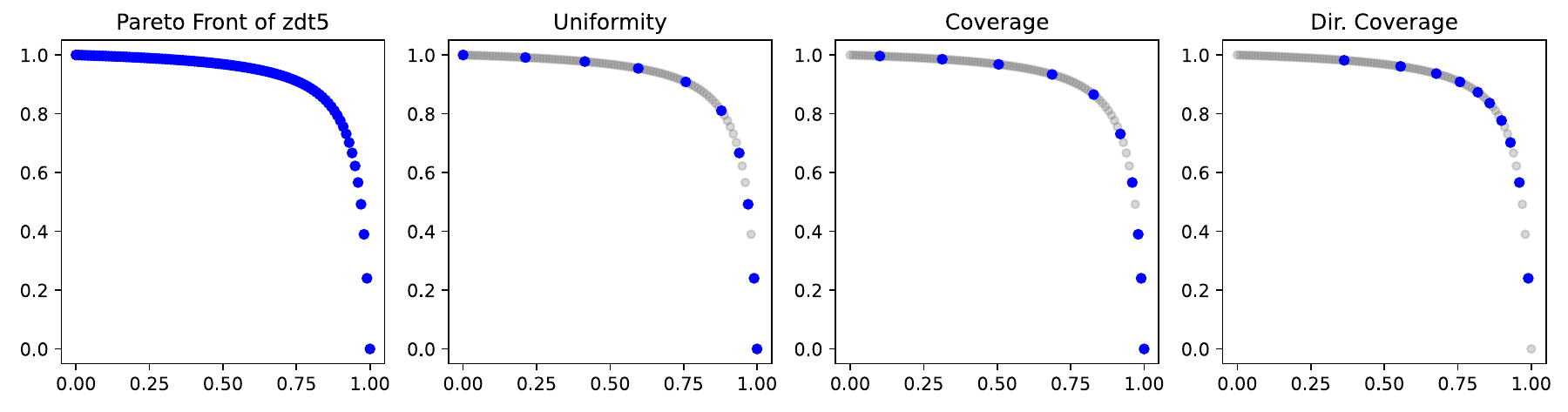}
  \label{fig:comp}
  \caption{The selected slates when optimizing Uniformity, Coverage and Directed Coverage
    on instance zdt3 with $k = 5$ and instance zdt5 with $k = 10$.
    Zdt3 illustrates the differences between measures. For uniformity, solutions are spaced
    as far apart as possible. For coverage, a central point in every
    cluster is selected. For directed coverage, a particularly efficient
    candidate in every cluster is selected.
    For zdt5, it is apparent that directed coverage puts more focus on
  covering the central options, which achieve a higher sum of objective values than those on the outside.}
\end{figure}

\begin{figure}[h]
  \centering
  \includegraphics[width=\linewidth]{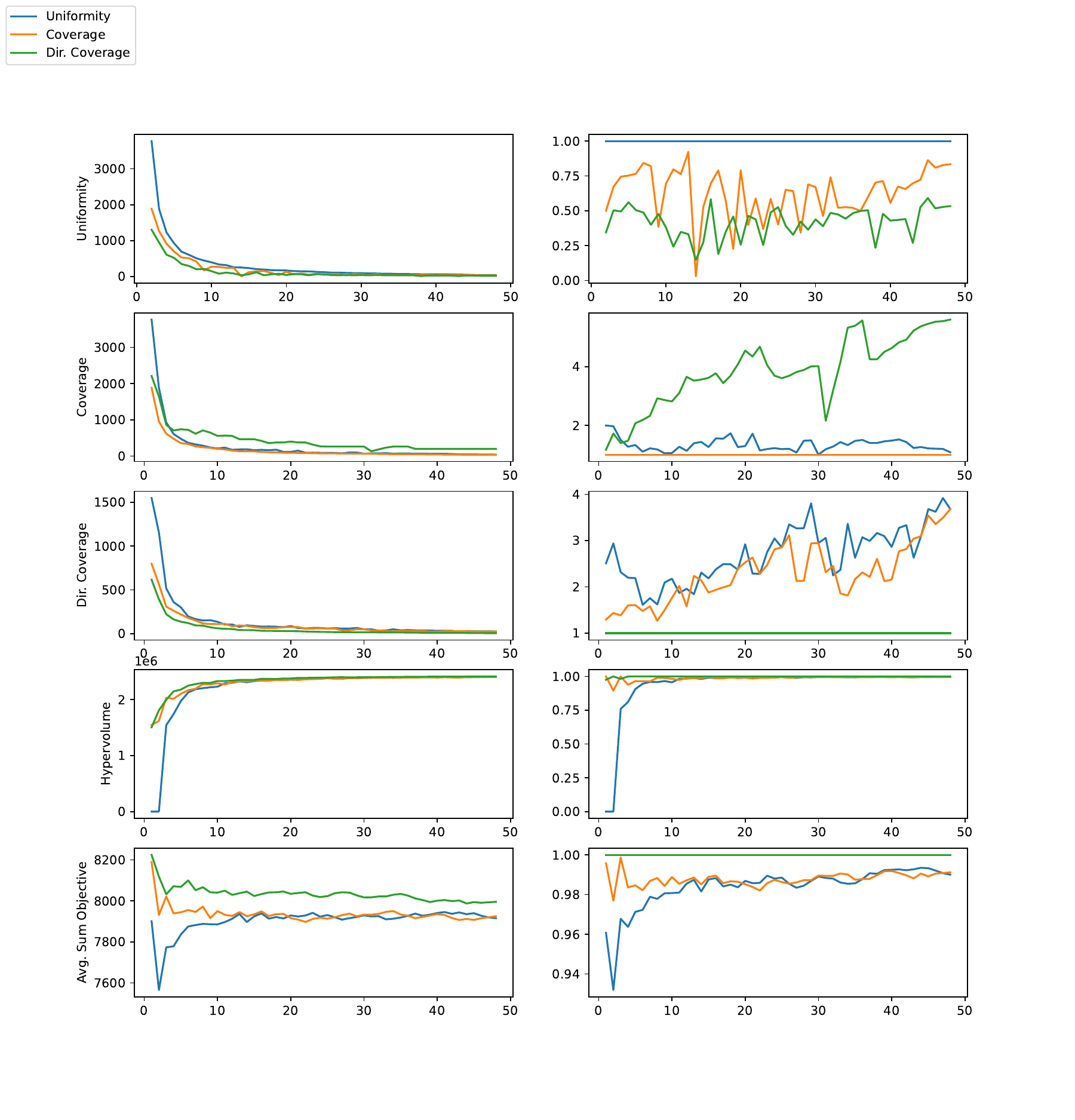}
  \caption{We evaluate the performance of optimal slates as determined
    by Uniformity, Coverage and Directed Coverage under the five
    performance performance measures discussed in \Cref{sec:exp} on
    instance PGMORL-Hopper-v2. For each $k \in [50]$ and measure $I$ we determine
    a slate $S \in \mathcal{S}(I, A, k)$ and then evaluate the slate using the performance measures. The left
    side displays the unscaled measurements of the measures.
    On the right side, each measure is rescaled such that the optimum slate always achieves a
    value of $1$. Note the significant improvement in solution quality for small $k$, which starts to
    stagnate quickly as $k$ increases. Also note, on the right hand side, the significant variations in
    performance with respect to a measure that is not explicitly optimized.
  }
\end{figure}

\end{document}